\documentclass{iopjournal}

\usepackage{graphicx}
\usepackage{placeins}
\usepackage{subcaption}
\usepackage{cite}
\usepackage{amsthm}
\usepackage{amsmath}
\usepackage{amssymb}
\usepackage{orcidlink}
\newtheorem{proposition}{Proposition}
\newtheorem{corollary}{Corollary}

\usepackage[linesnumbered,ruled,vlined]{algorithm2e}

\newcommand{\defineproblem}[2]{%
  \hypertarget{#1}{#2}%
}

\newcommand{\problemref}[1]{%
  Example~\hyperlink{#1}{#1}%
}

\newcommand{\problemreff}[1]{%
  \hyperlink{#1}{#1}%
}

\begin{document}

\articletype{Paper}

\title{Neural-network methods for two-dimensional finite-source reflector design}

\author{Roel Hacking$^{1,*}$\orcidlink{0009-0001-0820-2614}, Lisa Kusch$^1$\orcidlink{0000-0001-7558-0618}, Koondanibha Mitra$^1$\orcidlink{0000-0002-8264-5982}, Martijn Anthonissen$^1$\orcidlink{0000-0003-3667-769X} and Wilbert IJzerman$^{1,2}$}

\affil{$^1$Eindhoven University of Technology, PO Box 513, 5600 MB, Eindhoven, The Netherlands}

\affil{$^2$Signify, High Tech Campus 7, 5656 AE, Eindhoven, The Netherlands}

\affil{$^*$Author to whom any correspondence should be addressed.}

\email{r.g.j.hacking@tue.nl}


\begin{abstract}
We address the inverse problem of designing two-dimensional reflectors that transform light from a finite, extended source into a prescribed far-field distribution. We propose a neural-network parameterization of the reflector height and develop two differentiable objective functions: (i) a direct change-of-variables loss that pushes the source distribution through the learned inverse mapping, and (ii) a mesh-based loss that maps a target-space grid back to the source, integrates over intersections, and remains continuous even when the source is discontinuous. Gradients are obtained via automatic differentiation and optimized with a robust quasi-Newton method. As a comparison, we formulate a deconvolution baseline built on a simplified finite-source approximation: a one-dimensional monotone mapping is recovered from flux balance, yielding an ordinary differential equation solved in integrating-factor form; this solver is embedded in a modified Van~Cittert iteration with nonnegativity clipping and a ray-traced forward operator. Across four benchmarks---continuous and discontinuous sources, and with and without minimum-height constraints---we evaluate accuracy by the ray-traced normalized mean absolute error. Our neural-network approach converges faster and achieves consistently lower errors than the deconvolution method, and handles height constraints naturally. On our two main benchmarks the neural network reduces this error to $2\times 10^{-5}$ for a continuous source and $5\times 10^{-5}$ for a discontinuous source within a few seconds of wall-clock time on a single GPU, against $4\times 10^{-3}$ and $5\times 10^{-2}$ for the deconvolution baseline after several hundred seconds---two to three orders of magnitude lower error at two orders of magnitude less wall-clock time. We discuss how the method may be extended to rotationally symmetric and full three-dimensional settings via iterative correction schemes.

\end{abstract}

\section{Introduction}
Precise control of light propagation is a cornerstone of modern optics, underpinning applications that range from advanced illumination systems and high-efficiency solar concentrators to free-space optical communications and complex beam-shaping tasks\,\cite{Welford1989,Winston2004Nonimaging,Chaves2015IntroNonimaging}. Achieving a desired illumination pattern typically relies on carefully designed optical components---often reflectors or freeform surfaces---that redistribute and shape light according to prescribed specifications. Historically, reflector and freeform surface design has been carried out under simplified assumptions of point sources or collimated (infinite-distance) sources, leading to classic constructions such as parabolic mirrors and a rich body of nonimaging optics theory\,\cite{Welford1989,Winston2004Nonimaging,Chaves2015IntroNonimaging}. In practice, however, most light sources (e.g., LEDs or arc lamps) have finite spatial extent and angular emission. This finite-source nature introduces additional complexity into the optical design process.

Designing reflectors for extended sources presents challenges that go beyond the varying surface normal calculations required for collimated beams. In the ideal parallel-source limit, each point on the reflector intercepts a single ray direction; in the finite-source regime, however, every spatial point emits light over a continuous angular range. The resulting far-field intensity is effectively a kind of convolution of the ideal parallel-source response with this angular emission profile\,\cite{Boesel2019,Wei2021}. Due to étendue conservation, this angular extent introduces characteristic blurring, meaning that designs based on simple parallel-beam mappings often fail to realize high-contrast irradiance targets. While increasing the scale of the system relative to the source can reduce the angular size---approximating the collimated case---this compromises compactness\,\cite{Boesel2019}. These physical constraints have motivated the development of frameworks that explicitly incorporate source extent, building on point-source optimal-transport theory\,\cite{GlimmOliker2003JMS,GangboOliker2007COCV,Romijn2019JOSAA,Romijn2020JCP} and its recent extension to finite sources\,\cite{Benamou2022JCP}, and on edge-ray-mapping methods for extended sources\,\cite{Birch2020AO}. Other notable approaches include wavefront tailoring\,\cite{Sorgato2019Optica}, Simultaneous Multiple Surface (SMS) methods\,\cite{Benitez2004SMS3D,Winston2004Nonimaging}, and direct three-dimensional constructions for extended sources\,\cite{Wu2016,Zhu2022}.

A second, more signal-processing-oriented strategy treats the finite-source effect as a blurring convolution and applies deconvolution within the design loop. Supporting-ellipsoid methods combined with numerical optimization have been used to account for source extent\,\cite{Fournier2009}, and recent work shows that explicitly deblurring the extended-source response can yield freeform surfaces that are robust with respect to source size\,\cite{Wei2021}. Classical iterative deconvolution methods such as Van~Cittert and Richardson--Lucy provide algorithmic building blocks\,\cite{vanCittert1931,BurgerVanCittert1932,Richardson1972,Lucy1974,HillIoup1976,Jansson1997}. \emph{In this paper, our deconvolution-based baseline is not novel: it is an adaptation of existing approaches, most notably the method in \cite{vithesis}, tailored to our finite-source setting and evaluation protocol.} We include it to provide a meaningful and transparent point of comparison.

Concurrently, there is growing interest in differentiable and data-driven approaches for inverse illumination problems. Algorithmic differentiation and differentiable non-sequential ray tracing enable gradient-based optimization of freeform elements directly through the rendering pipeline\,\cite{Heemels2024OE,deKoning2023Arxiv}. These approaches treat the rendering pipeline as a stochastic black box and rely on Monte-Carlo sampling for the gradient; the methods we present below instead exploit the closed-form inverse mapping of the reflector---and, for the mesh-based loss, the cell-area structure of the target grid---to obtain a deterministic, smooth loss without sampling noise. Machine-learning surrogates have been explored for accelerating design or inferring topology from irradiance\,\cite{Gannon2019ML,Cerpentier2024OExNNTopo}, and neural parameterizations have been investigated for transport-driven partial differential equations (PDEs) and related reflector formulations\,\cite{HACKING2025100119}, with robust second-order quasi-Newton optimization improving performance\,\cite{URBAN2025113656}.

\textbf{This work.} We address the two-dimensional reflector-design problem for a finite linear source, transforming a given source luminance into a specified far-field irradiance. Our \emph{novel} contribution is a neural-network parameterization of the reflector height combined with \emph{two differentiable objective functions}:
(i) a direct change-of-variables loss that pushes the source distribution through the learned inverse mapping; and
(ii) a mesh-based loss that maps a target-space grid back to the source, integrates over valid intersections, and remains continuous even when the source is discontinuous.
Gradients are obtained via automatic differentiation and optimized with a robust quasi-Newton method. To contextualize these neural results, we compare against a \emph{baseline} deconvolution pipeline built on a simplified finite-source approximation and an iterative update (Van~Cittert with nonnegativity clipping), adapted from the approach in \cite{vithesis}. 

We begin by formulating the forward problem of light transport from a finite source to the far field via a reflector, establishing the relationship between source radiance, reflector geometry, and the resulting angular distribution. We then detail our two neural losses and training procedure, followed by the deconvolution baseline adapted from \cite{vithesis}. Two benchmark examples with known ``ground-truth'' reflectors and two additional height-constrained studies are used to evaluate performance.

\section{Problem setup}
We first describe the general problem considered in this paper and solved by our neural-network method in the following section. Next, we describe a simplification of this problem that is used as part of the deconvolution method against which we compare. 

\subsection{Finite-source problem setup}
\label{sec:finite-source-problem}
\begin{figure}[h]
\centering
\includegraphics{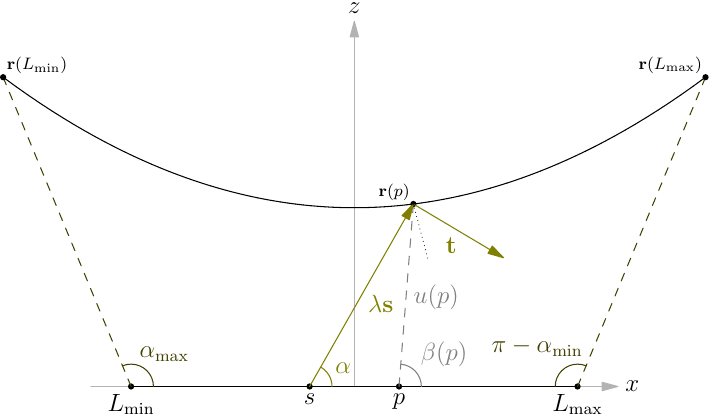}
\caption{Finite-source-to-far-field reflector system setup.}
\label{fig:finite-source-diagram}
\end{figure}
Consider the setup shown in Figure~\ref{fig:finite-source-diagram}. Let $\Omega = [L_{\min}, L_{\max}] \subset \mathbb{R}$ denote the spatial support of a linear light source situated along the $x$-axis in the $xz$-plane. For each $s \in \Omega$, light is emitted in all directions $\alpha \in A = [\alpha_{\min}, \alpha_{\max}] \subset \mathbb{R}$, where $A$ denotes the angular source domain represented as angles in radians measured counterclockwise from the positive $x$-axis. The full source domain is $\mathcal{S} := \Omega \times A$, and the light intensity is described by a nonnegative source distribution function $f : \mathcal{S} \to \mathbb{R}_{\geq 0}$.

Each ray is emitted in a direction $\mathbf{s} \in \mathbb{S}^1$ according to the angle $\alpha$. It is then redirected by a reflector surface defined over the same spatial base domain \( \Omega \) and a corresponding coordinate $p \in \Omega$. The reflector is given by the parametric map
\begin{equation}
\mathbf{r}(p) = \begin{bmatrix} p \\ 0 \end{bmatrix} + u(p) \begin{bmatrix} \cos\left( \beta(p) \right) \\ \sin\left( \beta(p) \right) \end{bmatrix}, 
\end{equation}
where \( \beta : \Omega \to \mathbb{R} \) is a function that maps source points to angles, and \( u : \Omega \to \mathbb{R}_{> 0} \) is the unknown height function describing the reflector geometry. For all experiments described here, we define $\beta(p)$ as 

\begin{equation}
\beta(p) = \alpha_{\max} + \frac{\alpha_{\min} - \alpha_{\max}}{L_{\max} - L_{\min}} (p - L_{\min}),
\end{equation}
i.e., we linearly map points from the range $[L_{\min}, L_{\max}]$ to the range $[\alpha_{\max}, \alpha_{\min}]$. 

This parameterization ensures that every emitted ray intersects the reflector curve, regardless of the choice of~$u$, provided only that $u$ is continuous and strictly positive. In particular, the result holds for any height function that a neural network $u_{\boldsymbol{\theta}}$ might produce during optimization, as long as $u_{\boldsymbol{\theta}}(p) > 0$ on $\Omega$. The proof is given in Appendix~\ref{app:ray-intersection}.

At each point, the curve normal is computed as
\begin{equation}
\mathbf{n}(p) = \frac{1}{\left\| \left( \frac{\partial r_z}{\partial p}, -\frac{\partial r_x}{\partial p} \right) \right\|} \left( \frac{\partial r_z}{\partial p}, -\frac{\partial r_x}{\partial p} \right),
\end{equation}
where \(\left\| \cdot \right\|\) denotes the Euclidean norm. 

A ray in direction \( \mathbf{s} \) reflects at the point \( \mathbf{r}(p) \) into the direction
\begin{equation}
\mathbf{t}(p) = \mathbf{s} - 2\langle \mathbf{s}, \mathbf{n}(p) \rangle \mathbf{n}(p).
\end{equation}
This reflected direction vector, denoted by its components $\mathbf{t}(p) = (t_x, t_z)$, is subsequently mapped to a scalar coordinate $\sigma$ on the far-field target domain $\Sigma = [T_{\min}, T_{\max}] \subset \mathbb{R}$ via stereographic projection. Geometrically, we project from the north pole $(0,1)$ of the unit circle onto the equatorial axis $z=0$, yielding the mapping
\begin{equation}
\sigma = \frac{t_x}{1 - t_z}.
\end{equation}
The full target domain is defined as \( \mathcal{T} := \Omega \times \Sigma \).

The reflector defines a mapping from \( \mathcal{S} \) to $\mathcal{T}$, implicitly through the geometry of the reflector. We denote this mapping as 

\begin{equation}
\mathbf{m} : \mathcal{S} \to \mathcal{T}, \quad (s, \alpha) \mapsto (p, \sigma).
\end{equation}

\noindent Concretely, the forward mapping $\mathbf{m}$ is obtained by (i) finding the intersection parameter $p$ such that the ray from $(s,0)$ in direction $(\cos\alpha, \sin\alpha)$ hits the reflector at $\mathbf{r}(p)$, (ii) reflecting the ray via the law of reflection using the surface normal $\mathbf{n}(p)$ to obtain $\mathbf{t}$, and (iii) applying the stereographic projection $\sigma = t_x/(1-t_z)$. This mapping has no closed-form expression, since step~(i) requires solving a nonlinear equation that depends on the reflector geometry.

The inverse mapping $\mathbf{m}^{-1}: \mathcal{T} \to \mathcal{S}$, $(p,\sigma) \mapsto (s, \alpha)$, reverses this process: (i) the far-field coordinate $\sigma$ is mapped back to the reflected direction $\mathbf{t}$ via the inverse stereographic projection $t_x = 2\sigma/(\sigma^2+1)$, $t_z = (\sigma^2-1)/(\sigma^2+1)$; (ii) the incoming direction $\mathbf{v}_{\mathrm{src}}$ is recovered by reversing the reflection at the known surface point $\mathbf{r}(p)$ with normal $\mathbf{n}(p)$; and (iii) the source coordinate $s$ is found analytically by intersecting the ray from $\mathbf{r}(p)$ in direction $\mathbf{v}_{\mathrm{src}}$ with the source line $z=0$, while $\alpha$ follows from the direction of $\mathbf{v}_{\mathrm{src}}$. Unlike the forward mapping, $\mathbf{m}^{-1}$ admits a closed-form expression, since the reflector parameter $p$ directly determines the surface point and normal, and all subsequent operations are analytical. The full expressions are given in Section~\ref{sec:direct-method}.

This induces a pushforward of the source distribution \( f \) to a directional output distribution over \( \Sigma \). The resulting marginal far-field distribution is given by
\begin{equation}
\label{eq:main}
g(\sigma) = \int_{\Omega} f\left( \mathbf{m}^{-1}(p, \sigma) \right) 
\left| \det\left( \left. \frac{\partial \mathbf{m}^{-1}(\mathbf{z})}{\partial \mathbf{z}} \right|_{\mathbf{z} = (p, \sigma)} \right) \right| \, \mathrm{d}p 
= \int_{\Omega} g\left(p, \sigma \right) \, \mathrm{d}p, 
\end{equation}
where $g(p, \sigma)$ is the full target distribution and $g(\sigma)$ is the far-field distribution obtained by marginalizing over the target domain. The determinant term accounts for the change of variables induced by the mapping \( \mathbf{m} \). Our objective is to determine the function \( u : \Omega \to \mathbb{R}_{>0} \) such that the induced target distribution \( g(\sigma) \) matches a prescribed far-field target distribution $\hat{g}(\sigma)$ for all \( \sigma \in \Sigma \).

\subsection{Approximate finite-source problem}
\label{sec:finite-source-approx-problem}
\begin{figure}[h]
\centering
\includegraphics{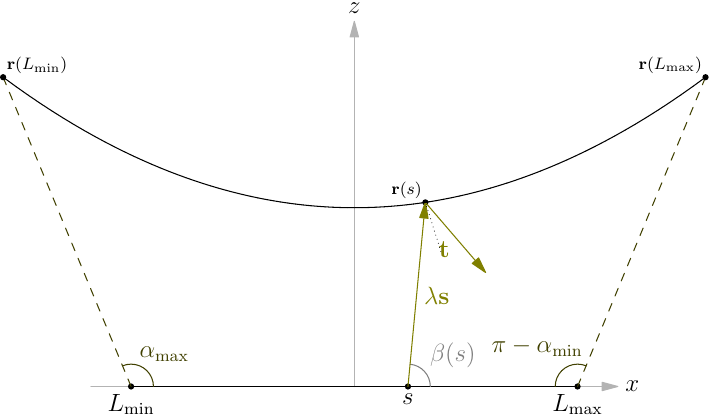}
\caption{Finite-source-to-far-field reflector system approximation setup.}
\label{fig:extended-source-diagram}
\end{figure}
To evaluate the performance of our neural-network method, we compare against a deconvolution-based method similar to that presented in \cite{vithesis}. This approach requires us to solve an approximate version of the full problem iteratively, refining the solution by adjusting the target distribution used for this subproblem. 

In the approximation of the problem, instead of each point emitting light in all directions $\alpha \in A$, light is instead emitted only in the direction $\beta(s)$. Thus, each ray hits the reflector at the point $\mathbf{r}(s)$ and the emission coordinate and reflector coordinate are the same, $p = s$. This eliminates the angular degree of freedom: the two-dimensional mapping $\mathbf{m} : (s, \alpha) \mapsto (p, \sigma)$ reduces to a one-dimensional mapping
\begin{equation}
m : \Omega \to \Sigma, \quad s \mapsto \sigma,
\end{equation}
since both $p$ and $\alpha$ are now determined by $s$ alone. Correspondingly, the source distribution reduces to its spatial marginal $f(s) = \int_A f(s, \alpha) \, \mathrm{d}\alpha$.

Under this reduction, the inverse mapping $\mathbf{m}^{-1}(p, \sigma)$ simplifies to $m^{-1}(\sigma)$, since for each $\sigma$ there is exactly one source coordinate $s = m^{-1}(\sigma)$ (with $p = s$). The one-dimensional analog of the change of variables in Eq.~\eqref{eq:main} then gives the far-field distribution directly, without integration:
\begin{equation}
g(\sigma) = f\left(m^{-1}(\sigma)\right) \left| \frac{d}{d\sigma} m^{-1}(\sigma) \right|.
\end{equation}

As both distributions are one-dimensional now, we obtain a much simpler ordinary differential equation (ODE), which can be solved both efficiently and accurately. Moreover, as we are using the same angular mapping function $\beta$ for this approximation of the problem, ray-tracing the solution we obtain from this approximation will not result in any rays missing the reflector curve, thus reflecting all light toward the target. A complementary reduction arises when the reflector height is scaled uniformly as $\lambda u(p)$ with $\lambda \to \infty$: the finite-source problem again collapses to a one-dimensional ODE, but driven by the angular marginal $F(\alpha) = \int_\Omega f(s, \alpha)\,\mathrm{d}s$ rather than the spatial marginal used here; see Appendix~\ref{app:scaling-limit} for details. This is the regime in which classical point-source and nonimaging-optics constructions\,\cite{Welford1989,Winston2004Nonimaging,Chaves2015IntroNonimaging,Romijn2019JOSAA} apply directly, and where the finite-source machinery developed in this paper---both the neural methods and the deconvolution baseline---becomes unnecessary.

The problem setup for this simplified approximation of the finite-source problem is illustrated in Figure~\ref{fig:extended-source-diagram}. All symbols used are summarized in Table~\ref{tab:symbols}.

\begin{table}[h]
\centering
\caption{Symbols used throughout the paper.}
\begin{tabular}{ll}
\hline
\textbf{Symbol} & \textbf{Meaning} \\
\hline
$\Omega$ & Spatial support of the linear light source (source base domain) \\
$A$ & Angular source domain (radian angles from positive $x$-axis) \\
$\mathcal{S}$ & Full source domain, $\mathcal{S} := \Omega \times A$ \\
$\Sigma$ & Stereographic far-field target domain \\
$\mathcal{T}$ & Full target domain, $\mathcal{T} := \Omega \times \Sigma$ \\
$f(s, \alpha)$ & Source distribution function over $\mathcal{S}$ \\
$f(s)$ & Marginal source distribution function over $\Omega$ \\
$g(p, \sigma)$ & Target distribution function over $\mathcal{T}$ \\
$g(\sigma)$ & Marginal far-field target distribution function over $\Sigma$ \\
$s$ & Source positional coordinate in $\Omega$ \\
$p$ & Reflector curve parameter in $\Omega$ \\
$\mathbf{s}$ & Emission direction vector \\
$\mathbf{r}(p)$ & Reflector surface point for parameter $p$ \\
$\beta(p)$ & Angular mapping function used to define the reflector curve $\mathbf{r}$ \\
$u(p)$ & Reflector height function \\
$\mathbf{n}(p)$ & Normal vector to the reflector surface at $p$ \\
$\mathbf{t}(p)$ & Reflected direction vector \\
$\sigma$ & Far-field angular coordinate in $\Sigma$ \\
$\mathbf{m}(s, \alpha)$ & Mapping from source to target domain \\
$\mathbf{m}^{-1}(p, \sigma)$ & Inverse of the mapping $\mathbf{m}$ \\
$m(s)$ & Simplified source-to-target mapping in the approximate problem \\
\hline
\end{tabular}
\label{tab:symbols}
\end{table}

\section{Methods}
\subsection{Neural-network-based solver}
\label{sec:ann}
We can solve the problem described in Section~\ref{sec:finite-source-problem} by representing the reflector height function $u(p)$ as a multilayer perceptron (MLP) and minimizing an appropriate loss function. We denote the network parameters (weights and biases) collectively as $\boldsymbol{\theta}$, and write the parameterized height function as $u_{\boldsymbol{\theta}}(p) = \mathrm{MLP}_{\boldsymbol{\theta}}(p) + p^2/2 + 1$, where the additive parabola biases the initial reflector toward a convex, strictly positive shape so the network only needs to learn a correction. The main challenge is to define a loss function such that the network learns a reflector that correctly maps the source distribution to the target distribution. We consider two approaches, described in the following sections. 

\subsubsection{Direct method}
\label{sec:direct-method}
The first method we will consider is to simply apply the change-of-variables formula shown in Eq.~\eqref{eq:main}. To do so, we must explicitly construct the inverse mapping \(\mathbf{m}_{\boldsymbol{\theta}}^{-1}: (p, \sigma) \to (s, \alpha)\) defined by the neural reflector surface. Given a reflector parameter \(p\) and a far-field coordinate \(\sigma\), the corresponding source coordinates are recovered via a geometric ``back-tracing'' procedure. 

First, the far-field scalar \(\sigma\) is mapped back to the unit reflected direction vector \(\mathbf{t} = (t_x, t_z)\) via the inverse stereographic projection:
\begin{equation}
    t_x = \frac{2\sigma}{\sigma^2 + 1}, \quad t_z = \frac{\sigma^2 - 1}{\sigma^2 + 1}.
\end{equation}
Simultaneously, the surface point \(\mathbf{r}(p) = (r_x, r_z)\) and its unit normal \(\mathbf{n}(p)\) are computed from the network output \(u_{\boldsymbol{\theta}}(p)\) and its spatial gradients (obtained via automatic differentiation). By treating \(\mathbf{t}\) as the outgoing ray and reversing the law of reflection, we determine the unit vector \(\mathbf{v}_{\mathrm{src}} = (v_x, v_z)\) pointing from the reflector toward the source:
\begin{equation}
    \mathbf{v}_{\mathrm{src}} = 2 \langle \mathbf{t}, \mathbf{n}(p) \rangle \mathbf{n}(p) - \mathbf{t}.
\end{equation}
We then cast a ray from \(\mathbf{r}(p)\) in the direction \(\mathbf{v}_{\mathrm{src}}\) to find its intersection with the source line (the \(x\)-axis, where \(z=0\)). The spatial source coordinate \(s\) is found by solving for the scalar \(\lambda\) such that \(r_z + \lambda v_z = 0\):
\begin{equation}
    s = r_x + \lambda v_x = r_x - r_z \frac{v_x}{v_z}.
\end{equation}
Finally, the emission angle is recovered as \(\alpha = \operatorname{atan2}(-v_z, -v_x)\). If \(s \notin \Omega\), the density contribution is zero.

With \((s, \alpha)\) determined, we can evaluate the full target distribution \(g_{\boldsymbol{\theta}}(p, \sigma)\) via the determinant of the Jacobian of this inverse mapping. To obtain the marginal target distribution \(g_{\boldsymbol{\theta}}(\sigma)\) corresponding to the neural network, we must approximate the integral
\begin{equation}
  g_{\boldsymbol{\theta}}(\sigma) = \int_{L_{\min}}^{L_{\max}} g_{\boldsymbol{\theta}}(p,\sigma)\,\mathrm{d}p.
  \label{eq:marginal-integral}
\end{equation}
As this is a simple one-dimensional integral, we can approximate it using either a quadrature rule or some alternative method depending on whether we expect discontinuities or non-smoothness in the integrand. Based on this marginal distribution, we can then define a loss function:
\begin{equation}
\mathcal{L}(\boldsymbol{\theta}) = \int_{T_{\min}}^{T_{\max}} \left(g_{\boldsymbol{\theta}}(\sigma) - \hat{g}(\sigma)\right)^2\,\mathrm{d}\sigma.
\label{eq:main-loss-fn}
\end{equation}
This loss depends only on the prescribed marginal far-field target $\hat{g}(\sigma)$. The same is true of the mesh-based loss introduced in the next section, and of the deconvolution baseline of Section~\ref{sec:iterative-deconv}: none of the methods compared in this paper requires the full target distribution $\hat{g}(p,\sigma)$.

Once again, we cannot directly compute the above integral. Thus, we sample \(n_\sigma\) equally spaced points \(\sigma_i\) in the interval \([T_{\min}, T_{\max}]\) to approximate the integral in Eq.~\eqref{eq:main-loss-fn}, and use \(n_p\) equally spaced points \(p_j\) to approximate the integral in Eq.~\eqref{eq:marginal-integral} for each $\sigma_i$.

Both discretizations need some attention. The far-field count $n_\sigma$ must be large enough to resolve the prescribed $\hat{g}$; in our examples we match it to the binning used for ray-traced evaluation. The $p$-count requires more care, because the integrand $g_{\boldsymbol{\theta}}(p, \sigma)$ is supported only on the subset of $(p, \sigma)$ values whose inverse image $\mathbf{m}_{\boldsymbol{\theta}}^{-1}(p, \sigma)$ lies inside the source domain $\mathcal{S}$. The angular extent of the source controls the thickness of this support in $p$ for each $\sigma$; in the parallel-source limit ($\alpha_{\max} - \alpha_{\min} \to 0$) the support collapses to a one-dimensional curve and any finite discretization in $p$ generally misses it. For the extended sources considered here, a few dozen $p$-samples suffice; the empirical sensitivity of the method to both counts is examined in the hyperparameter sensitivity analysis below.

Using automatic differentiation (AD), we can compute both the Jacobian required for the derivation of \(g_{\boldsymbol{\theta}}(\sigma)\), as well as the gradient of the loss function. This allows us to employ gradient-based optimization techniques to minimize the loss and obtain an approximate solution to the finite-source problem. The specific optimizer we use here is the quasi-Newton method described in \cite{URBAN2025113656}.

\subsubsection{Mesh-based method}
\label{sec:mesh-method}
When the function \( f \) is continuous, the previously introduced loss function performs reliably. However, when \( f \) exhibits discontinuities---particularly at the boundaries of the domain---these discontinuities propagate into the loss function. As a consequence, standard gradient-based optimization algorithms are no longer applicable, due to the lack of differentiability. While one could resort to gradient-free optimization techniques such as evolutionary strategies, these methods typically suffer from slower convergence and reduced accuracy. To address this issue, we propose a reformulation of the loss function that ensures continuity, even when \( f \) is discontinuous.

\subsubsection*{Mesh-based integration approach}

Instead of directly integrating over the function \( g_{\boldsymbol{\theta}} \), we subdivide the target space into a regular grid of quadrilaterals, forming a mesh, see Figure~\ref{fig:mesh_transformation}(a). Let \( \{Q_j\}_{j=1}^P \) denote these quadrilateral cells in the target space. Each vertex \( \mathbf{y}_{jk} \in Q_j \) is mapped back to the source domain using the inverse map \( \mathbf{m}_{\boldsymbol{\theta}}^{-1} \), which is parameterized, for instance, by a neural network representing a reflector surface. Assuming the mapping \( \mathbf{m}_{\boldsymbol{\theta}}^{-1} \) is smooth, we approximate the image of each quadrilateral \( Q_j \) under this map by another quadrilateral \( \tilde{Q}_j \subset \mathcal{S} \), whose vertices are given by the mapped vertices \( \mathbf{x}_{jk} = \mathbf{m}_{\boldsymbol{\theta}}^{-1}(\mathbf{y}_{jk}) \). Figure~\ref{fig:mesh_transformation}(b) shows such a mapped mesh. 

We further assume that any discontinuities in the source distribution \( f(\mathbf{x}) \) occur only at the boundary of the source domain \( \mathcal{S} \). Then, for each mapped quadrilateral \( \tilde{Q}_j \), we compute its intersection \( \Omega_j = \tilde{Q}_j \cap \mathcal{S} \) with the source domain. Within this intersection \( \Omega_j \), we draw a fixed number of samples \( \{\mathbf{x}_{ji}\}_{i=1}^{N_j} \), and approximate the integral of \( f \) over \( \Omega_j \) via numerical quadrature:
\begin{equation}
\tilde{I}_j = \int_{\Omega_j} f(\mathbf{x}) \, \mathrm{d}\mathbf{x} \approx \frac{|\Omega_j|}{N_j} \sum_{i=1}^{N_j} f(\mathbf{x}_{ji}),
\end{equation}
where \( |\Omega_j| \) is the area of the intersection \( \Omega_j \), estimated using geometric algorithms (e.g., polygon intersection and area computation). This is shown in Figures~\ref{fig:mesh_transformation}(c) and (d). 

Since each \( \Omega_j \) corresponds to a unique target-space quadrilateral \( Q_j \), the resulting \( \tilde{I}_j \) values represent approximate integrals of the source distribution pushed forward to the target space, giving a binned two-dimensional approximation of the target distribution:
\begin{equation}
\tilde{g}_{\boldsymbol{\theta}}(\mathbf{y}) \approx \sum_{j=1}^P \tilde{I}_j \cdot \chi_{Q_j}(\mathbf{y}),
\end{equation}
where \( \chi_{Q_j} \) denotes the indicator function over \( Q_j \). This two-dimensional binned approximation is shown in Figure~\ref{fig:mesh_transformation}(e); summing \( \tilde{I}_j \) along the \( p \)-axis of the mesh yields the corresponding binned marginal far-field target distribution, shown in Figure~\ref{fig:mesh_transformation}(f).

To define a loss function that is continuous with respect to \( \boldsymbol{\theta} \), we compute a discrepancy metric---such as the mean squared error (MSE) or a Wasserstein-type distance---between this binned marginal \( \tilde{g}_{\boldsymbol{\theta}}(\sigma) \) and the desired far-field target distribution \( \hat{g}(\sigma) \), evaluated over the same \( \sigma \)-binning. Due to the smoothness of the inverse mapping \( \mathbf{m}_{\boldsymbol{\theta}}^{-1} \) and the use of integration over compact intersections \( \Omega_j \), this loss function is continuous with respect to \( \boldsymbol{\theta} \), even when \( f \) is discontinuous.

As in the direct method, two discretizations enter the mesh-based loss: the number $n_\sigma^{\mathrm{mesh}}$ of target-space cells along the $\sigma$-axis and the number $n_p^{\mathrm{mesh}}$ along the $p$-axis, giving an $n_p^{\mathrm{mesh}} \times n_\sigma^{\mathrm{mesh}}$ mesh. These count cells of the target-space grid and so are distinct from the sample counts $n_\sigma$, $n_p$ in the direct method. Unlike the direct method, the mesh integration does not degenerate when the support of $g_{\boldsymbol{\theta}}(p, \sigma)$ becomes thin in $p$: every cell $Q_j$ contributes the geometric intersection $\Omega_j = \tilde{Q}_j \cap \mathcal{S}$, which is either empty or a polygon of finite area. Excessively coarse $\sigma$-meshes can still fail, since the loss is then satisfied bin-by-bin without constraining the pointwise far field; the empirical sensitivity of the method to both counts is examined in the hyperparameter sensitivity analysis below.

\begin{figure}[htbp]
    \centering
    \includegraphics{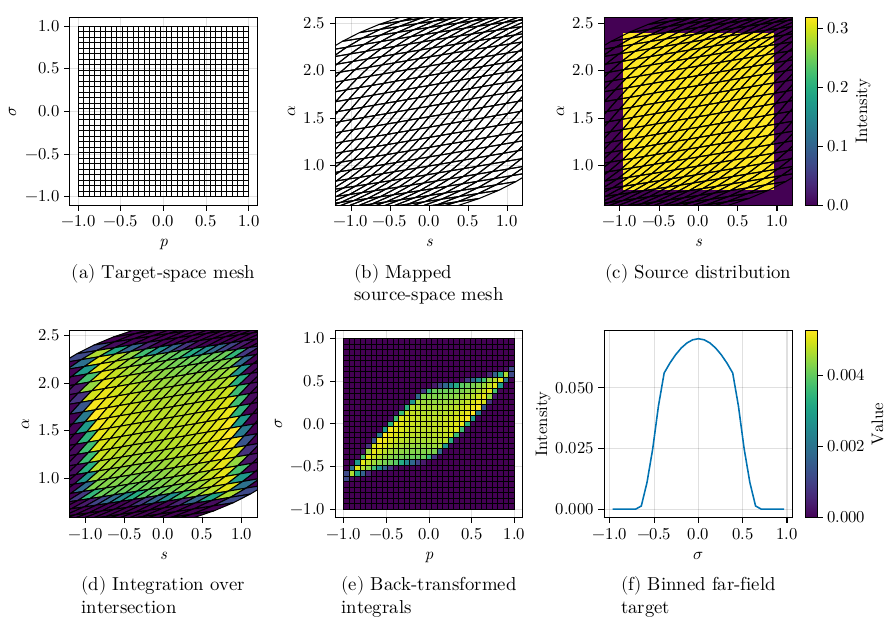}
    \caption{Mesh-based procedure for computing a continuous loss for inverse reflector design. (a) The target space is discretized into quadrilateral cells. (b) Each vertex is mapped to the source space. (c) The source distribution is evaluated. (d) Integration is performed over valid intersections with the source domain. (e) The integrals are assigned back to the original mesh cells. (f) A histogram of the far-field output is obtained.}
    \label{fig:mesh_transformation}
\end{figure}

\subsection{Iterative deconvolution}
\label{sec:iterative-deconv}
As a baseline against which to compare the neural-network solution of the previous section, we now adapt the deconvolution-based method of \cite{vithesis} to our finite-source setting. As noted in the introduction, this construction is not part of our contribution: it is a representative member of the established class of feedback/deconvolution illumination-design methods, and we include it to provide a meaningful and transparent point of comparison. The method requires a fast solver for an approximate version of the full problem, which is then used iteratively to refine the solution. The problem approximation we use here is the one described in Section~\ref{sec:finite-source-approx-problem}. We now derive an expression for the solution of this approximate problem.

\subsubsection*{ODE derivation}
We impose the constraint that the mapping derived from the reflector profile $u$ should be monotonically increasing. As such, the net flux from \(\,[L_{\min},\,s]\) must match that from \(\,[T_{\min},\,m(s)]\), as no flux after $s$ can be mapped before $m(s)$.  Equivalently,
\begin{equation}
\label{eq:flux-match-simple}
  \int_{\,L_{\min}}^{\,s} f(\tau)\,d\tau
  \;=\;
  \int_{\,T_{\min}}^{\,m(s)} g(\xi)\,d\xi.
\end{equation}
Differentiating both sides with respect to \(s\) gives
\[
  f(s)
  \;=\;
  g\!\bigl(m(s)\bigr)\,m'(s)
  \quad\Longrightarrow\quad
  m'(s)
  \;=\;
  \frac{\,f(s)\,}{\,g\bigl(m(s)\bigr)\!}.
\]
We also have \(m(L_{\min})=T_{\min}\).  Hence
\begin{equation}
\label{eq:flux-match2}
  \begin{cases}
    m'(s) \;=\; \dfrac{f(s)}{\,g\bigl(m(s)\bigr)\!},\\
    m(L_{\min}) \;=\; T_{\min}.
  \end{cases}
\end{equation}
Solving this ODE or root-finding $m(s)$ to satisfy Eq.~\eqref{eq:flux-match-simple} yields a monotonic mapping \(s \mapsto m(s)\) that redistributes the flux from \(f\) into \(g\). 

\subsubsection*{Law of reflection and the reflected direction}
For a ray emitted at $s$, we know that the intersection point will be \(\bigl(r_x(s),\,r_z(s)\bigr)\). We define the incident direction as
\begin{equation}
  \mathbf{s}(s)
  \;=\;
  \bigl(\cos\beta(s),\;\sin\beta(s)\bigr).
\end{equation}
To determine the reflection vector, we first compute the derivatives
\begin{align}
\begin{split}
\label{eq:reflector-derivs}
  r_x'(s)
    \;&=\; 
    1 - \beta'(s)\,\sin\beta(s)\,u(s) \;+\;\cos\beta(s)\,u'(s), \\
  r_z'(s)
    \;&=\;
    \beta'(s)\,\cos\beta(s)\,u(s) \;+\;\sin\beta(s)\,u'(s).
\end{split}
\end{align}
Let the tangent vector be \(\boldsymbol{\tau}(s)=(r_x'(s),\,r_z'(s))\) with norm
\(\|\boldsymbol{\tau}(s)\|=\sqrt{\,r_x'(s)^{2}+r_z'(s)^{2}\,}\).
The unit normal (pointing down) is
\begin{equation}
  \mathbf{n}(s)
  \;=\;
  \frac{1}{\,\|\boldsymbol{\tau}(s)\|\,}\,\bigl(r_z'(s),\,-r_x'(s)\bigr).
\end{equation}
By mirror reflection,
\begin{equation}
  \mathbf{t}(s)
  \;=\;
  \mathbf{s}(s)
  \;-\;
  2\,\bigl[\mathbf{s}(s)\cdot \mathbf{n}(s)\bigr]\,\mathbf{n}(s),
  \quad
  \mathbf{t}(s)
  = \bigl(t_x(s),\,t_z(s)\bigr).
\end{equation}
We then stereographically project:
\begin{equation}
  \sigma_{\mathrm{geom}}(s)
  \;=\;
  \frac{\,t_x(s)\,}{\,1 - t_z(s)\,}.
\end{equation}

We now want to find an expression for $u'(s)$ such that we can integrate (numerically) and obtain a $u(s)$ that realizes the mapping $m(s)$ derived from Eq.~\eqref{eq:flux-match-simple} or Eq.~\eqref{eq:flux-match2}.

\subsubsection*{Geometric derivation of $u'(s)$}
Given a particular $s$ and mapping $\sigma = m(s)$, we can compute the corresponding reflection direction as

\begin{equation}
  t_x = \frac{2\sigma}{\sigma^2 + 1},
  \quad
  t_z = \frac{\sigma^2 - 1}{\sigma^2 + 1}.
\end{equation}

We can also compute the emission direction as

\begin{equation}
  s_x = \cos\beta(s),
  \quad
  s_z = \sin\beta(s).
\end{equation}

Based on these two vectors, we can reconstruct the normal of the reflector at $s$ by averaging and rescaling the two vectors:

\begin{equation}
  \mathbf{n} = \frac{\mathbf{t} - \mathbf{s}}{\left\|\mathbf{t} - \mathbf{s}\right\|}.
\end{equation}

This normal gives us the scaled derivatives (\(\hat{r}_x' = -\mathbf{n}_z\), \(\hat{r}_z' = \mathbf{n}_x\)), which are related to the original derivatives as
\begin{equation}
r_x' = \kappa \hat{r}_x' \quad \text{and} \quad r_z' = \kappa \hat{r}_z'.
\end{equation}
This implies that multiplying the scaled derivatives by \(\kappa\) recovers the original derivatives.

Using Eq.~\eqref{eq:reflector-derivs}, we have the system
\begin{align}
\kappa \hat{r}_x'(s) &= -\sin(\beta(s)) \cdot \beta'(s) \cdot u(s) + \cos(\beta(s)) \cdot u'(s) + 1, \label{eq1} \\
\kappa \hat{r}_z'(s) &= \cos(\beta(s)) \cdot \beta'(s) \cdot u(s) + \sin(\beta(s)) \cdot u'(s), 
\label{eq2}
\end{align}

\noindent where $\kappa$ and $u'(s)$ are the unknowns. Finally, solving for $u'(s)$ we get

\begin{equation}
\label{eq:u-prime-geom}
    u'(s) = \frac{- \hat{r}_{x}'(s) \beta'(s) u \cos{\left(\beta(s) \right)} - \hat{r}_{z}'(s) \beta'(s) u \sin{\left(\beta(s) \right)} + \hat{r}_{z}'(s)}{\hat{r}_{x}'(s) \sin{\left(\beta(s) \right)} - \hat{r}_{z}'(s) \cos{\left(\beta(s) \right)}}. 
\end{equation}

\noindent We can then pick an arbitrary $h$ and formulate the initial value problem of solving Eq.~\eqref{eq:u-prime-geom} with $u(L_{\min}) = h$. Note that, depending on the choice of $h$, the resulting reflector profile $u$ may take on negative values, which is physically undesirable. To address this, an outer loop may be required to adjust $h$ until a nonnegative $u$ is obtained.

\subsubsection*{Integrating factors solution}
We can now substitute the expressions for $\hat{r}_x'(s)$ and $\hat{r}_z'(s)$ into Eq.~\eqref{eq:u-prime-geom} to obtain (again omitting explicit $(s)$ for readability)
\begin{equation}
    u' \;=\; \frac{- \beta' \sigma^{2} u \cos\beta \;+\; 2 \beta' \sigma u \sin\beta \;+\; \beta' u \cos\beta \;+\; \sigma^{2} \cos\beta \;-\; 2 \sigma \;+\; \cos\beta}{\sigma^{2} \sin\beta \;-\; \sigma^{2} \;+\; 2 \sigma \cos\beta \;-\; \sin\beta \;-\; 1}.
\end{equation}
We rewrite this in linear form
\begin{equation}
    u'(s) \;+\; a(s)\,u(s) \;=\; b(s),
\end{equation}
where
\begin{align}
    a(s) \,&=\, \frac{\beta'\!\left(- \sigma^{2} \cos\beta \;+\; 2 \sigma \sin\beta \;+\; \cos\beta\right)}{- \sigma^{2} \sin\beta \;+\; \sigma^{2} \;-\; 2 \sigma \cos\beta \;+\; \sin\beta \;+\; 1},\\[0.3em]
    b(s) \,&=\, \frac{- \sigma^{2} \cos\beta \;+\; 2 \sigma \;-\; \cos\beta}{- \sigma^{2} \sin\beta \;+\; \sigma^{2} \;-\; 2 \sigma \cos\beta \;+\; \sin\beta \;+\; 1}.
\end{align}
The integrating-factor solution is
\begin{equation}
\label{eq:approx-sol-expr}
    u(s) \;=\; \frac{1}{\mu(s)}\!\left[
        h \;+\; \int_{L_{\min}}^{s} b(\xi)\,\mu(\xi)\, \mathrm{d}\xi
    \right],
    \qquad
    \mu(s) \;=\; \exp\!\left( \int_{L_{\min}}^{s} a(\xi)\, \mathrm{d}\xi \right).
\end{equation}
We then approximate the integrals numerically (we use cumulative Simpson's integration). Moreover, if only the initial value $h$ is changed, the precomputed $\mu$ and the cumulative integral can be reused, which simplifies enforcing constraints such as $\min u \ge 0$ (see Fig.~\ref{fig:extended-source}).

\begin{figure}
    \centering
    \includegraphics{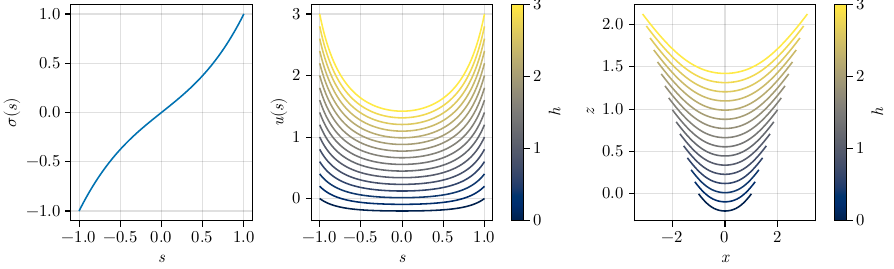}
    \caption{Reflectors for the mapping $m(s) = (2s + s^3)/3$ for different values of $h$ computed using Eq.~\eqref{eq:approx-sol-expr}. The left plot shows the mapping, the center plot shows the obtained functions $u$ for different values of $h$, and the right plot shows the corresponding reflectors in the $xz$-plane.}
    \label{fig:extended-source}
\end{figure}

\subsubsection*{Iterative procedure}
Now that we have a method for solving the approximate finite-source problem, we can derive a deconvolution algorithm for solving the full problem. Specifically, we define an operator $P\bigl[\tilde{g}\bigr]$ that we assume to be approximately a convolution, where $\tilde{g}$ denotes a `virtual' far-field target distribution used to attain our true desired target distribution $\hat{g}$. This operator performs three steps:

\begin{enumerate}
    \item It solves the approximated finite-source problem described, using the marginal distribution $f(s)$ of the original problem as the source distribution of the subproblem, and the input distribution $\tilde{g}$ as the target distribution of the subproblem. The source distribution integral is also adjusted such that it matches the given target distribution integral. $h$ is chosen such that the minimum of $u(s)$ is $h_{\min}>0$. 
    \item It ray-traces, based on the original finite-source system, the reflector computed in the first step. We specifically use quasi-Monte Carlo ray tracing, with $2^{19}$ rays and $64$ bins in the far-field target domain. 
    \item It normalizes and resamples the obtained ray-traced image, as ray-tracing gives us the approximated integrals over bins, whereas we are interested in the function values at the same sample points as $\tilde{g}$. 
\end{enumerate}

Our goal is to find a virtual target distribution $\tilde{g}$ such that $P\bigl[\tilde{g}\bigr]\approx \hat{g}$. The reflector computed by the operator $P$ is then considered the approximate solution to the full finite-source problem. We could use any generic deconvolution algorithm for this---as long as it requires only evaluations of $P\bigl[\tilde{g}\bigr]$---but we chose to use Van~Cittert deconvolution here, as it does not require us to divide by $P\bigl[\tilde{g}\bigr]$, which may be (close to) zero for some points. One downside of this algorithm is that it may result in partially negative target distributions. As such, we have modified the algorithm to clip the target distribution $\tilde{g}^{(n)}$ to be nonnegative. See Algorithm~\ref{alg:van-cittert} for the full deconvolution algorithm.

\begin{algorithm}
\SetAlgoLined
\KwData{
  Desired distribution $\hat{g}(\sigma)$;\\
  Forward operator $P(\cdot)$ (which ``convolves'' a candidate $\tilde{g}$ to the finite-source outcome);\\
  Initial guess $\tilde{g}^{(0)}(\sigma) \ge 0$;\\
  Learning rate $\eta \in (0,1]$;\\
  Maximum number of iterations $N$.
}
\KwResult{Updated distribution $\tilde{g}^{(N)}(\sigma)$ that approximates $\hat{g}(\sigma)$ under $P(\cdot)$.}

\For{$n \leftarrow 0$ \KwTo $N-1$}{
  \textbf{Step 1:} \emph{Compute residual} 
  $r^{(n)}(\sigma) \leftarrow \hat{g}(\sigma)\;-\;P\bigl[\tilde{g}^{(n)}\bigr](\sigma).$\\[4pt]
  
  \textbf{Step 2:} \emph{Tentative update} 
  $g_{\mathrm{temp}}^{(n+1)}(\sigma) 
    \;\leftarrow\;
    \tilde{g}^{(n)}(\sigma)\;+\;\eta\,r^{(n)}(\sigma).$\\[4pt]
  
  \textbf{Step 3:} \emph{Clip negativity} 
  $\tilde{g}^{(n+1)}(\sigma)
    \;\leftarrow\;
    \max\!\Bigl(0,\;g_{\mathrm{temp}}^{(n+1)}(\sigma)\Bigr)\,.$
}

\caption{Modified Van~Cittert deconvolution ensuring nonnegativity by clipping.}
\label{alg:van-cittert}
\end{algorithm}

\section{Numerical examples}
We now present four numerical examples. The first example uses a continuous source, hence we can use the method described in Section~\ref{sec:direct-method} and compare it against the deconvolution method of Section~\ref{sec:iterative-deconv}. For the second example, we consider a discontinuous source, and thus use the method described in Section~\ref{sec:mesh-method} instead. For the second pair of examples, we study the effect of constraining the height of the optimized reflector on the final accuracy for both the neural-network and deconvolution methods. 

To quantify the performance of both the neural-network method and the deconvolution, we ray-trace the reflectors obtained by both methods. We then compute the Normalized Mean Absolute Error (NMAE) between the ray-traced results and the desired target distribution. The NMAE is defined as

\begin{equation}
    \mathrm{NMAE} = \frac{\frac{1}{N}\sum_{i=1}^{N} \left| \hat{b}_i - b_i \right|}{\frac{1}{N}\sum_{i=1}^{N} \left| \hat{b}_i \right|}
\end{equation}

\noindent where $N$ is the total number of bins, $\hat{b}_i$ is the value of the $i$-th bin of the desired far-field target distribution, and $b_i$ is the value of the $i$-th ray-traced bin. The NMAE is a measure of how well the ray-traced results match the desired target distribution, with lower values indicating better performance.

\paragraph{Implementation and optimizer settings.}
All experiments use JAX with the Flax neural-network library, with float32 matmul precision, on a single NVIDIA RTX~4090 GPU. The optimizer is the self-scaled Broyden quasi-Newton method of \cite{URBAN2025113656} with a strong Wolfe zoom line search ($c_1 = 10^{-3}$, $c_2 = 0.7$). For the mesh-based method we use a two-phase warm start within a 16~s budget: the first third optimizes against a Gaussian-smoothed source, softening the boundary discontinuity to make the (already-continuous) mesh loss smoother and ease early convergence; the remainder uses the exact source.

\subsection*{Example A: Continuous source}
\defineproblem{A}{}
\begin{figure}[htbp]
  \centering

  \begin{subfigure}[t]{0.38\textwidth}
    \centering
    \includegraphics[width=\linewidth]{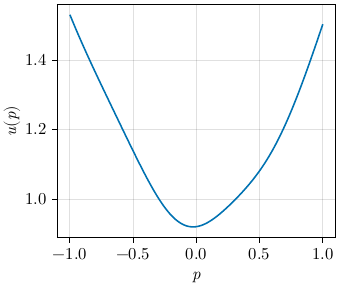}
    \caption{Ground-truth reflector}
    \label{fig:denver-reflector}
  \end{subfigure}
  \hspace{0.01\textwidth}
  \begin{subfigure}[t]{0.38\textwidth}
    \centering
    \includegraphics[width=\linewidth]{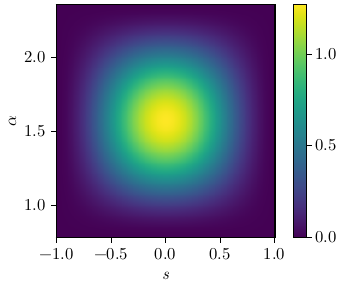}
    \caption{Source distribution}
    \label{fig:denver-source}
  \end{subfigure}

  \vspace{0.02\textheight}

  \begin{subfigure}[t]{0.38\textwidth}
    \centering
    \includegraphics[width=\linewidth]{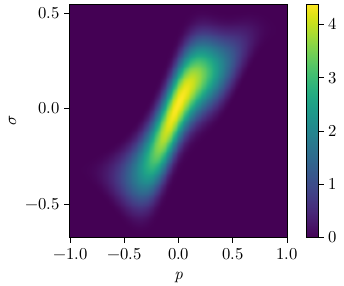}
    \caption{Target distribution derived from (a) and (b)}
    \label{fig:denver-target}
  \end{subfigure}
  \hspace{0.01\textwidth}
  \begin{subfigure}[t]{0.38\textwidth}
    \centering
    \includegraphics[width=\linewidth]{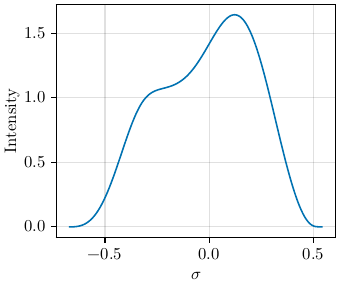}
    \caption{Marginal far-field target distribution derived from (c)}
    \label{fig:denver-target-marginal}
  \end{subfigure}

  \caption{Ground-truth reflector, source distribution, corresponding target distribution, and far-field target distribution for \protect\problemref{A}.}
  \label{fig:example-a-problem}
\end{figure}

In this example, we consider a continuous source defined over the interval $[L_{\min}, L_{\max}] = [-1, 1]$ with angular bounds $[\alpha_{\min}, \alpha_{\max}] = [\frac{\pi}{4}, \frac{3\pi}{4}]$. The source intensity decays smoothly to zero at the boundaries, enabling the application of the direct method described in Section~\ref{sec:direct-method}.

To ensure the existence of a solution, we generate a ground-truth reflector represented as a convex polyharmonic spline. From this, we derive the target distribution and the corresponding far-field target distribution, as illustrated in Figure~\ref{fig:example-a-problem}. By using a known ground-truth reflector, any errors in the computed solutions can be attributed solely to optimization inaccuracies rather than physical constraints of the system.

Under typical conditions, only the source distribution and the far-field target distribution are available. Accordingly, both the neural-network and deconvolution methods were provided access only to this information. Both methods utilized a grid of 64 samples in the far-field target domain $\Sigma$. For the neural-network method, we employed a multilayer perceptron (MLP) with two hidden layers, each containing 24 nodes, and used a squared hyperbolic tangent activation function for both layers. For the deconvolution method, we set the learning rate parameter $\eta = 0.5$.

The results of applying both methods are shown in Figure~\ref{fig:example-a-results}. The left panel displays the optimized reflectors for both methods together with the ground truth (dashed). The neural network's reflector closely matches the ground truth, while the deconvolution method exhibits significant deviations near the edges. The middle panel shows the pointwise absolute deviation $|u_{\mathrm{pred}}(p) - u_{\mathrm{gt}}(p)|$ for each method: the neural-network reflector stays near $10^{-4}$ across the domain, whereas the deconvolution reflector deviates by up to $\sim 6\times 10^{-2}$. Since the ground-truth reflector may not be unique, however, pointwise deviation alone does not necessarily indicate poor performance. We therefore also evaluate the ray-tracing error, shown in the right panel as a function of optimization time. This confirms that the final solution obtained by the deconvolution method is worse than that obtained by the neural network. Moreover, the neural-network algorithm here converges before the deconvolution method has performed even a single iteration. The corresponding best-NMAE values, time-to-best, and iteration counts are listed in Table~\ref{tab:example-ab-summary} alongside those of \problemref{B}.

\begin{figure}[htbp]
  \centering
  \begin{subfigure}[t]{0.32\textwidth}
    \centering
    \includegraphics[width=\linewidth]{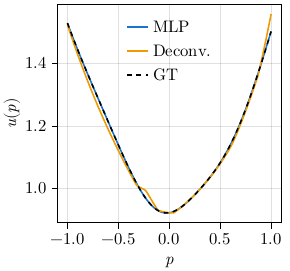}
    \caption{Optimized reflectors}
    \label{fig:denver-reflectors}
  \end{subfigure}\hspace{0.01\textwidth}%
  \begin{subfigure}[t]{0.32\textwidth}
    \centering
    \includegraphics[width=\linewidth]{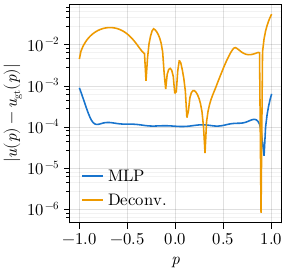}
    \caption{Absolute error vs.\ ground truth}
    \label{fig:denver-reflector-error}
  \end{subfigure}\hspace{0.01\textwidth}%
  \begin{subfigure}[t]{0.32\textwidth}
    \centering
    \includegraphics[width=\linewidth]{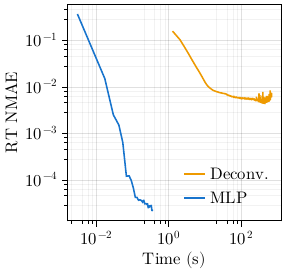}
    \caption{Ray-tracing NMAE vs.\ time}
    \label{fig:denver-convergence}
  \end{subfigure}

  \caption{Final reflectors, pointwise absolute reflector error, and convergence history for the neural-network and deconvolution methods applied to \protect\problemref{A}. In (a) the ground-truth reflector is drawn as a dashed black curve on top of the predictions so that the otherwise-coincident neural-network curve remains visible. Panel (b) gives the pointwise absolute deviation from ground truth.}
  \label{fig:example-a-results}
\end{figure}

\subsection*{Example B: Discontinuous source}
\defineproblem{B}{}
\begin{figure}[htbp]
  \centering

  \begin{subfigure}[t]{0.38\textwidth}
    \centering
    \includegraphics[width=\linewidth]{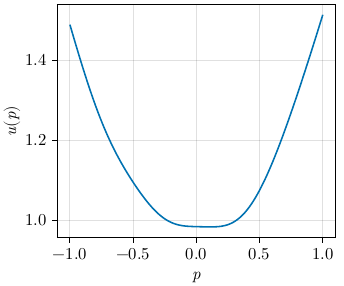}
    \caption{Ground-truth reflector}
    \label{fig:denver2-reflector}
  \end{subfigure}
  \hspace{0.01\textwidth}
  \begin{subfigure}[t]{0.38\textwidth}
    \centering
    \includegraphics[width=\linewidth]{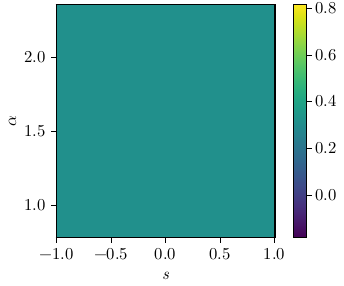}
    \caption{Source distribution}
    \label{fig:denver2-source}
  \end{subfigure}

  \vspace{0.02\textheight}

  \begin{subfigure}[t]{0.38\textwidth}
    \centering
    \includegraphics[width=\linewidth]{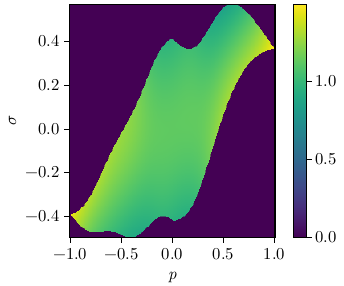}
    \caption{Target distribution derived from (a) and (b)}
    \label{fig:denver2-target}
  \end{subfigure}
  \hspace{0.01\textwidth}
  \begin{subfigure}[t]{0.38\textwidth}
    \centering
    \includegraphics[width=\linewidth]{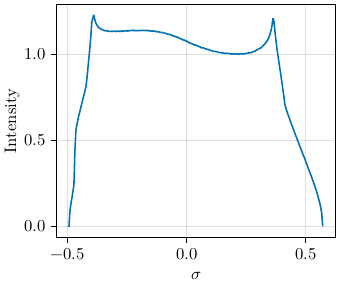}
    \caption{Marginal far-field target distribution derived from (c)}
    \label{fig:denver2-target-marginal}
  \end{subfigure}

  \caption{Ground-truth reflector, source distribution, corresponding target distribution, and far-field target distribution for \protect\problemref{B}.}
  \label{fig:example-b}
\end{figure}

For the second example, we construct a reflector problem by selecting a random ground-truth reflector and specifying a source. Unlike the first example, we use a uniform source, resulting in a discontinuity at the boundaries of the source domain, as illustrated in Figure~\ref{fig:example-b}.

As discussed in Section~\ref{sec:mesh-method}, the direct method applied in \problemref{A} is not suitable for this case due to the discontinuity. Instead, we employ the mesh-based method. All other parameters for both the neural-network and deconvolution methods remain identical to those in the first example, including a grid of 64 samples in the far-field target domain $\Sigma$, a multilayer perceptron with two hidden layers of 24 nodes each using a squared hyperbolic tangent activation function for the neural network, and a learning rate of $\eta = 0.5$ for the deconvolution method.

The results are shown in Figure~\ref{fig:example-b-results}. As in the first example, the neural-network method yields a more accurate solution and converges more rapidly than the deconvolution method. The pointwise reflector error in panel (b) shows that the neural-network reflector stays near $3\times 10^{-4}$ across the domain, whereas the deconvolution reflector deviates by up to $\sim 8\times 10^{-2}$---the larger error from the deconvolution baseline relative to \problemref{A} is consistent with the harder, discontinuous-source setting. Summary numerics for both examples are collected in Table~\ref{tab:example-ab-summary}.

\begin{figure}[htbp]
  \centering
  \begin{subfigure}[t]{0.32\textwidth}
    \centering
    \includegraphics[width=\linewidth]{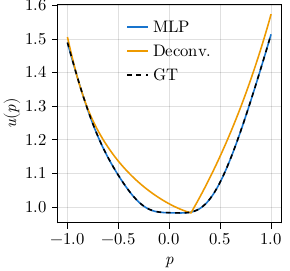}
    \caption{Optimized reflectors}
    \label{fig:denver2-reflectors}
  \end{subfigure}\hspace{0.01\textwidth}%
  \begin{subfigure}[t]{0.32\textwidth}
    \centering
    \includegraphics[width=\linewidth]{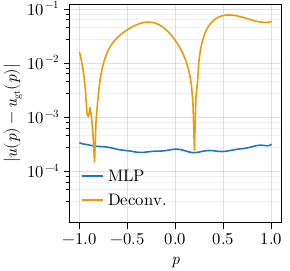}
    \caption{Absolute error vs.\ ground truth}
    \label{fig:denver2-reflector-error}
  \end{subfigure}\hspace{0.01\textwidth}%
  \begin{subfigure}[t]{0.32\textwidth}
    \centering
    \includegraphics[width=\linewidth]{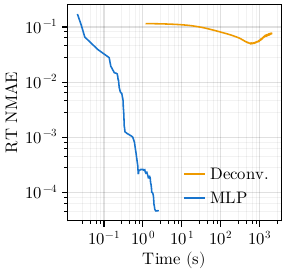}
    \caption{Ray-tracing NMAE vs.\ time}
    \label{fig:denver2-convergence}
  \end{subfigure}

  \caption{Final reflectors, pointwise absolute reflector error, and convergence history for the neural-network and deconvolution methods applied to \protect\problemref{B}. In (a) the ground-truth reflector is drawn as a dashed black curve on top of the predictions so that the otherwise-coincident neural-network curve remains visible. Panel (b) gives the pointwise absolute deviation from ground truth.}
  \label{fig:example-b-results}
\end{figure}

\begin{table}[htbp]
  \centering
  \caption{Summary of accuracy and computational cost for the neural-network (NN) and deconvolution (Deconv) methods on \protect\problemref{A} and \protect\problemref{B}. ``Best RT NMAE'' is the smallest ray-traced NMAE attained over the run, ``time-to-best'' is the wall-clock time at which that minimum is reached, ``Iterations'' is the corresponding number of accepted/applied updates, and ``Final loss'' is the value of $\mathcal{L}(\boldsymbol{\theta})$ at the end of training for the neural method (not defined for the deconvolution baseline). All runs are on a single NVIDIA RTX~4090 GPU.}
  \label{tab:example-ab-summary}
  \begin{tabular}{llllll}
    \hline
    Example & Method & Best RT NMAE & Time-to-best & Iterations & Final loss \\
    \hline
    A & NN     & $2.16\times 10^{-5}$ & $0.34$~s & $209$  & $1.29\times 10^{-11}$ \\
    A & Deconv & $4.44\times 10^{-3}$ & $410$~s  & $900$  & --- \\
    B & NN     & $4.65\times 10^{-5}$ & $2.54$~s & $278$  & $6.74\times 10^{-12}$ \\
    B & Deconv & $5.02\times 10^{-2}$ & $611$~s  & $3000$ & --- \\
    \hline
  \end{tabular}
\end{table}

\subsection*{Example C: Continuous source with height penalty}
\defineproblem{C}{}
In the previous two examples, the reflector's height was unconstrained. However, as the reflector is positioned further from the source, the problem increasingly resembles a point-source scenario, which (1) reduces the need for a complex finite-source approach and (2) may violate real-world physical constraints. Designers typically prefer a finite-source approach to avoid placing the reflector arbitrarily far from the source.

For the second pair of examples, we investigate the impact of reflector height on ray-tracing NMAE. We adopt the same ground-truth reflector, source distribution, and derived target distribution as in \problemref{A}, as shown in Figure~\ref{fig:example-a-problem}. To enforce a height constraint, we modify the neural network's loss function by adding the following term:
\begin{equation}
\label{eq:height-penalty}
    \mathcal{L}_{\mathrm{height}} = \left( \min_{p \in [L_{\min}, L_{\max}]} u(p) - h_{\min} \right)^2,
\end{equation}
where $h_{\min}$ is a user-defined minimum-height constraint. This term penalizes deviations of the reflector's minimum height from $h_{\min}$. For the deconvolution method, we set the initial height $h$ such that the reflector's minimum is at $h_{\min}$. Both methods are then evaluated across various $h_{\min}$ values, and their ray-tracing NMAE is compared.

The results are presented in Figure~\ref{fig:example-c-results}, which plots the NMAE for both methods as a function of $h_{\min}$. A vertical line marks the $h_{\min}$ corresponding to the ground-truth reflector. The neural-network method consistently outperforms the deconvolution method, quickly converging to an accurate solution at and above the ground-truth height. The deconvolution method improves when increasing the height, but never matches the performance of the neural network. Note, however, that the neural network generally performs worse than when unconstrained (see Figure~\ref{fig:example-a-results}), presumably as a result of the increased complexity of the loss function and optimization problem introduced by the additional loss term.

\begin{figure}
    \centering
    \includegraphics[width=0.49\linewidth]{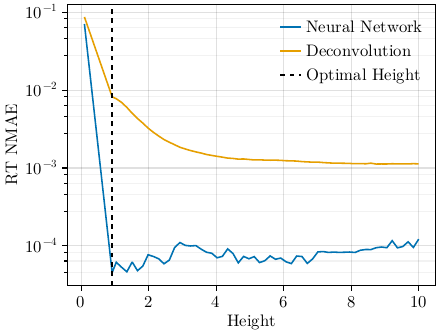}
    \caption{NMAE for the neural-network and deconvolution methods as a function of minimum reflector height $h_{\min}$ for \protect\problemref{C}.}
    \label{fig:example-c-results}
\end{figure}

\subsection*{Example D: Uniform target with height penalty}
\defineproblem{D}{}
In the previous examples, we utilized problems with a known ground-truth reflector. In practice, however, such a reflector is typically unavailable, making it uncertain whether a solution exists. For this example, we consider a problem without a known ground-truth reflector. Specifically, we use the same source distribution as in \problemref{A} but adopt a uniform far-field target distribution. 

We conduct experiments analogous to those in \problemref{C}, applying both the neural-network and deconvolution methods to this problem and comparing their performance across different values of the height constraint $h_{\min}$. The neural network incorporates the height constraint Eq.~\eqref{eq:height-penalty}, while the deconvolution method initializes the reflector height such that its minimum is at $h_{\min}$. The MLP architecture and deconvolution learning rate match those of \problemref{A}, but we use a denser 127-sample grid in the far-field target domain $\Sigma$ to resolve the uniform target.

The results are presented in Figure~\ref{fig:example-d-results}. We can see that the performance of both methods is generally much worse than in previous examples, which is to be expected as the target distribution used here is likely physically unattainable. Here, for the first time, the deconvolution method slightly outperforms the neural network for smaller values of $h_{\min}$, though the difference is small, and the speed of the neural-network method is still much greater (the same order of magnitude as in \problemref{A}). When we increase the height, the neural-network method starts performing better again, though the difference between both methods remains relatively small. This is likely the result of both methods already being close to the optimal error achievable by a reflector at that height.

One possible explanation for the crossover at low $h_{\min}$ is the different way the two methods enforce the height constraint. The deconvolution method satisfies $\min u = h_{\min}$ exactly by construction (through the choice of the initial value $h$ in the ODE), whereas the neural network relies on a soft penalty term (Eq.~\eqref{eq:height-penalty}), turning the optimization into a multi-objective problem. Balancing competing loss terms is a well-known challenge in physics-informed neural-network training\,\cite{Wang2021PINN}, and at low heights---where the constraint is tightly binding and the problem is already difficult due to strong finite-source blurring---this additional burden may slightly disadvantage the neural network. We note, however, that the effect is small and that the precise mechanism remains an open question.

\begin{figure}
    \centering
    \includegraphics[width=0.49\linewidth]{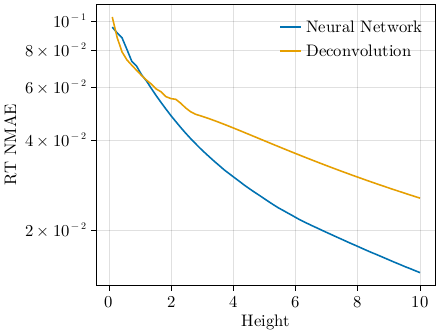}
        \caption{NMAE for the neural-network and deconvolution methods as a function of minimum reflector height $h_{\min}$ for \protect\problemref{D}.}
    \label{fig:example-d-results}
\end{figure}

\subsection*{Hyperparameter sensitivity}
To complement the single-run results above, we conduct a one-factor sensitivity sweep of the two neural-network methods over the four hyperparameters identified in Sections~\ref{sec:direct-method} and~\ref{sec:mesh-method} as the most relevant. For the direct method these are MLP depth, MLP width, the number of $\sigma$-samples $n_\sigma$ in the outer-integral discretization, and the number of $p$-samples $n_p$ in the inner-integral discretization. For the mesh method the analogous four are MLP depth, MLP width, the number of cells $n_\sigma^{\mathrm{mesh}}$ along the $\sigma$-axis of the target grid, and the number of cells $n_p^{\mathrm{mesh}}$ along the $p$-axis. For each method we re-use the corresponding example setup (\problemref{A} for the direct change-of-variables loss; \problemref{B} for the mesh-based loss) and run five seeds per setting, with five settings per axis: depth $\in \{1,2,3,4,5\}$; width $\in \{8, 16, 24, 32, 48\}$; for the direct method $n_\sigma \in \{16, 32, 64, 128, 256\}$ and $n_p \in \{8, 16, 32, 64, 128\}$; for the mesh method $n_\sigma^{\mathrm{mesh}}, n_p^{\mathrm{mesh}} \in \{16, 32, 64, 96, 128\}$. This gives 100 trials per method. Ray-tracing NMAE is evaluated with $2^{28}$ quasi-Monte Carlo rays and 63 bins.

\begin{figure}[htbp]
  \centering
  \includegraphics[width=0.99\textwidth]{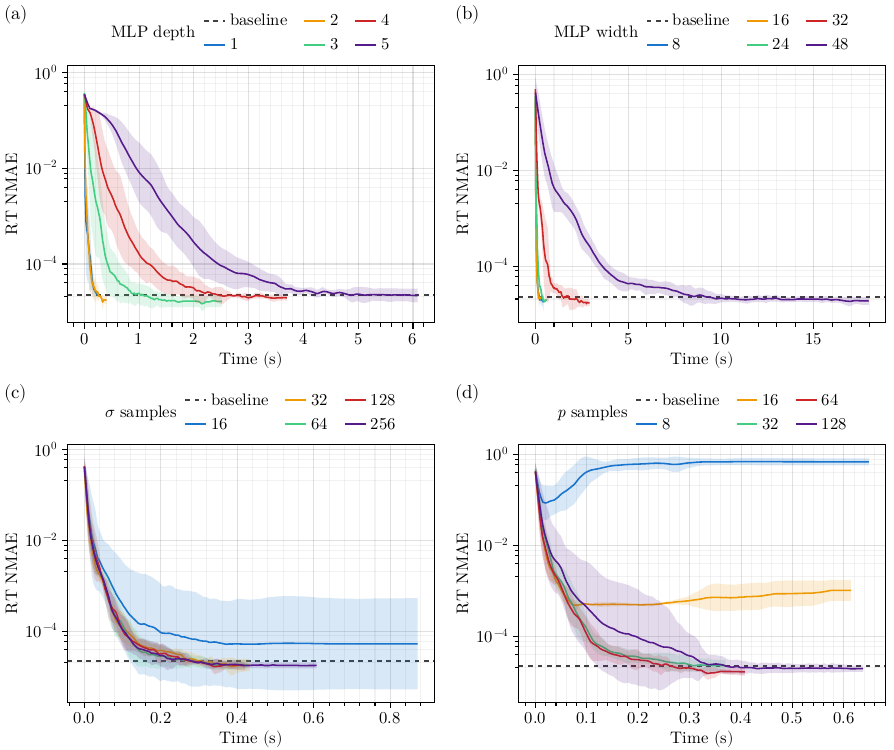}
  \caption{Hyperparameter sensitivity for the direct change-of-variables method on \protect\problemref{A}: ray-tracing NMAE as a function of optimization time for one-factor sweeps over (a)~MLP depth, (b)~MLP width, (c)~far-field sample count $n_\sigma$, and (d)~$p$-sample count $n_p$ (five settings $\times$ five seeds each). Solid curves are per-setting means in log space; shaded bands are 95\,\% confidence intervals. The dashed line is the single-run baseline from Figure~\ref{fig:example-a-results}.}
  \label{fig:cov-sensitivity}
\end{figure}

\begin{figure}[htbp]
  \centering
  \includegraphics[width=0.99\textwidth]{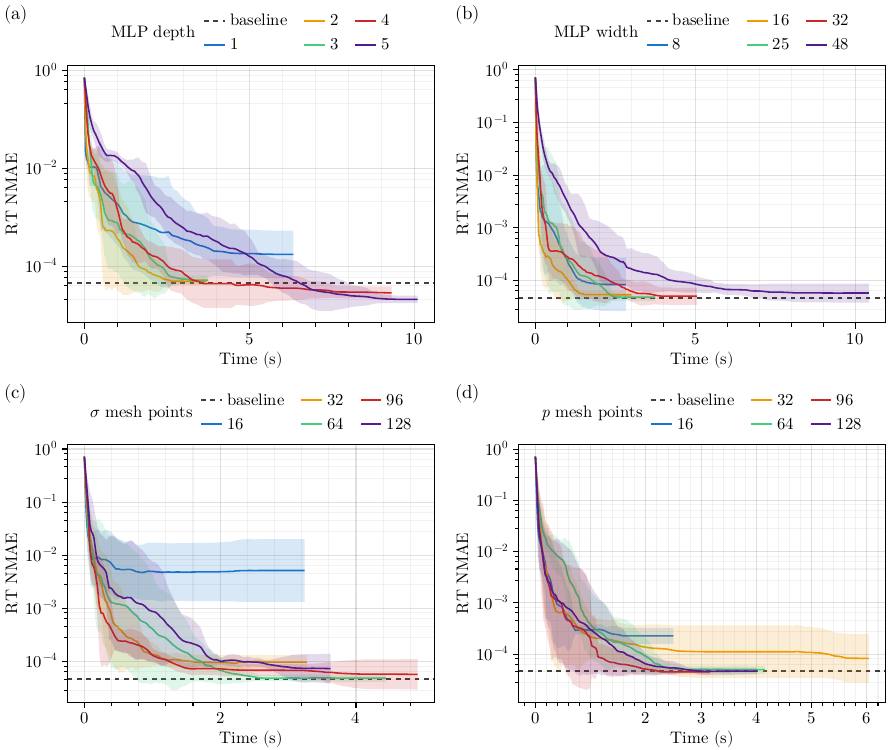}
  \caption{Hyperparameter sensitivity for the mesh-based method on \protect\problemref{B}: ray-tracing NMAE as a function of optimization time for one-factor sweeps over (a)~MLP depth, (b)~MLP width, (c)~$\sigma$-mesh resolution $n_\sigma^{\mathrm{mesh}}$, and (d)~$p$-mesh resolution $n_p^{\mathrm{mesh}}$ (five settings $\times$ five seeds each). Solid curves are per-setting means in log space; shaded bands are 95\,\% confidence intervals. The dashed line is the single-run baseline from Figure~\ref{fig:example-b-results}.}
  \label{fig:mesh-sensitivity}
\end{figure}

The direct change-of-variables method (Figure~\ref{fig:cov-sensitivity}) is essentially insensitive to network depth, network width, and $n_\sigma$ (panels (a), (b), (c)): across all settings tested, the final ray-tracing NMAE remains between roughly $1.6\times 10^{-5}$ and $2.5\times 10^{-5}$. Wall-clock time, by contrast, scales sharply with both depth and width---width-$8$ networks converge in $\sim 0.33$\,s and width-$48$ networks in $\sim 10.7$\,s, broadly consistent with the quadratic growth of the dense quasi-Newton Hessian update with the number of parameters. The $p$-sample count $n_p$ (panel (d)) is the only hyperparameter with a hard failure mode in this study: at $n_p \le 16$ the marginal integration in Eq.~\eqref{eq:marginal-integral} is too coarse to reconstruct the target, in line with the support-thickness argument given in Section~\ref{sec:direct-method}.

The mesh-based method (Figure~\ref{fig:mesh-sensitivity}) shows similar overall stability with two distinct trends. First, deeper networks (panel (a)) slightly improve the final NMAE (median $\sim 2\times 10^{-5}$ at depth $5$ versus $\sim 5\times 10^{-5}$ at depth $2$) with little change in wall-clock time, since per-iteration cost is dominated by the mesh-loss assembly rather than the network forward/backward. Second, $n_\sigma^{\mathrm{mesh}} = 16$ (panel (c)) is a catastrophic failure mode: all five seeds yield ray-tracing NMAE in the range $2$--$18 \times 10^{-3}$, the only such failure observed in either study. This is the mechanism anticipated in Section~\ref{sec:mesh-method}: when the $\sigma$-mesh is too coarse, the cell-aggregated marginal MSE can be satisfied without constraining the pointwise far field. The $n_p^{\mathrm{mesh}} = 16$ runs (panel (d)) are degraded but still usable. Taken together, both methods are robust across a wide range of hyperparameter settings and only fail when the discretization is too coarse to resolve the marginal target.

\section{Discussion}
The results presented in the previous section demonstrate the effectiveness of the neural-network-based method compared to the deconvolution baseline for solving finite-source reflector design problems. The neural-network approach consistently delivers more accurate solutions in less time. Both differentiable loss formulations---the direct change-of-variables loss and the mesh-based loss for discontinuous sources---prove effective when combined with the MLP parameterization and quasi-Newton optimization. Furthermore, the neural-network method enables optimization of the reflector's minimum height, as shown in Examples \problemreff{A} and \problemreff{B}, and supports the imposition of height constraints, as demonstrated in Examples \problemreff{C} and \problemreff{D}.

Two aspects of the methods help explain the gap in accuracy and speed. First, the neural methods minimize a differentiable scalar loss with a quasi-Newton optimizer that uses a strong Wolfe line search; the line search never accepts a step that increases the loss, so the training loss is monotonically non-increasing by construction. Since the loss is a faithful proxy for the ray-traced NMAE in all of our experiments, the NMAE follows the loss downward. The deconvolution iteration carries no such guarantee: it is a fixed-step Van~Cittert update with no built-in monotonicity in either the flux-balance residual or the ray-traced NMAE, and its inner approximate solver---a one-dimensional model that ignores the angular extent of the source---drives the reflector toward a fixed point of an approximate forward operator rather than the true one.

The wall-clock gap has a separate cause. Each neural iteration evaluates the closed-form inverse map, a small integration or mesh-assembly kernel, and an automatic-differentiation gradient, with no ray tracing in the loop. The deconvolution loop, in contrast, requires a high-budget ray trace at every step to evaluate the forward operator, and so pays a per-iteration cost orders of magnitude larger. Together with the monotonicity argument above, this explains both the lower final NMAE and the faster convergence observed in Examples~\problemreff{A}--\problemreff{D}. The two neural losses themselves are complementary: the mesh-based loss exists specifically to keep the loss continuous when the source is discontinuous, so that quasi-Newton remains applicable, while for smooth sources the simpler direct loss reaches comparable accuracy at lower per-iteration cost.

Extending the direct method from Section~\ref{sec:direct-method} to full three-dimensional applications is conceptually straightforward, as the mathematical framework can be readily adapted. Without modifications, the computational cost may increase significantly due to the higher dimensionality of the domains $\mathcal{S}$ and $\mathcal{T}$, which become four-dimensional in three-dimensional problems. Given the efficiency of the current two-dimensional implementation, the extension to three dimensions remains computationally feasible, though to further enhance scalability, stochastic optimization could be adopted: rather than evaluating all target points at each iteration, randomly subsampling these points reduces the per-iteration cost substantially while preserving effective descent. The quasi-Newton optimizer used here does not support stochasticity, necessitating alternative optimization strategies, such as those discussed in \cite{moritz2016linearly}. Common machine-learning optimizers, such as Adam \cite{kingma2017adammethodstochasticoptimization}, have proven inadequate for physics-informed neural-network training \cite{HACKING2025100119,jnini2024gaussnewtonnaturalgradientdescent,URBAN2025113656}, a category our approach might be placed under.

The mesh-based method from Section~\ref{sec:mesh-method}, while theoretically extendable to three dimensions, poses practical challenges. Computing intersections between quadrilaterals in two dimensions is manageable, but performing analogous operations in four-dimensional space is significantly more complex. To address this, alternative approaches could be considered, such as smoothing the source distribution with a gradually decreasing smoothing factor during optimization, employing optimization techniques for discontinuous loss functions \cite{Kreikemeyer_2023}, or developing formulations that inherently preserve continuity.

One natural application domain for the two-dimensional method presented here is the design of rotationally symmetric three-dimensional reflectors. In such systems, the two-dimensional reflector profile can be interpreted as a meridional cross-section, and the corresponding three-dimensional reflector is obtained by revolving this profile around the optical axis. Under the assumption that only meridional rays---those lying in planes containing the axis of symmetry---contribute to the far field, the three-dimensional design problem reduces to the two-dimensional problem solved here. However, this approximation neglects skew rays, which do not pass through the axis of symmetry and can carry significant energy, particularly for spatially extended sources. The accuracy of the meridional-only approximation therefore depends on the source geometry and the ratio of source extent to reflector distance.

To account for skew rays and improve the fidelity of rotationally symmetric designs, a hybrid iterative approach could be considered. The deconvolution baseline in this paper (Algorithm~\ref{alg:van-cittert}) already demonstrates the underlying principle: an approximate inner solver (the ODE-based method) produces a reflector from a simplified model, and Van~Cittert iteration adjusts a virtual target distribution to compensate for the approximation error, using ray tracing of the full two-dimensional finite-source model as the forward operator. The same iterative correction principle could bridge the gap between two and three dimensions: the two-dimensional MLP solver from Section~\ref{sec:ann} would serve as the inner solver, its output profile would be revolved to generate a rotationally symmetric three-dimensional reflector, and a full three-dimensional ray trace---including skew rays---would serve as the forward operator. Van~Cittert iteration would then update the virtual target for the next two-dimensional solve, progressively compensating for the error introduced by the meridional-only approximation. Because the two-dimensional MLP solver converges rapidly and produces accurate reflectors, it constitutes a strong inner solver for such an iterative scheme. Extending this idea beyond rotational symmetry to fully freeform three-dimensional reflectors would require parameterizing the reflector as a two-dimensional surface rather than a one-dimensional profile, and remains an open challenge.

\section{Conclusions}
We have presented a comprehensive comparison of two approaches for designing two-dimensional reflectors to transform light from a finite source into a prescribed far-field illumination pattern: a neural-network-based method using a multilayer perceptron (MLP) for reflector shape parameterization and a semi-analytical iterative deconvolution method based on a simplified finite-source approximation. Through numerical experiments, including cases with continuous and discontinuous source distributions, as well as unconstrained and height-constrained scenarios, the neural-network method consistently outperformed the deconvolution approach in terms of accuracy, as measured by the Normalized Mean Absolute Error (NMAE), and convergence speed. The neural network's flexibility in representing complex reflector geometries and its ability to incorporate constraints, such as minimum reflector height, make it particularly effective for addressing the challenges of finite-source reflector design.

The results also highlight the robustness of the neural-network method across diverse problem conditions, including discontinuities in the source distribution, where the mesh-based loss formulation ensures continuity of the loss function and enables effective optimization. In contrast, the deconvolution method, while computationally efficient for simplified approximations, struggles with accuracy and stability, particularly in the presence of source discontinuities. These findings suggest that the neural-network approach is better suited for applications requiring precise control over far-field illumination from finite sources.

Looking ahead, the differentiable framework developed here opens several avenues for future work. As discussed in the previous section, the formulation extends naturally to three-dimensional reflector design, and a hybrid iterative approach combining the two-dimensional MLP solver with full three-dimensional ray tracing could enable accurate design of rotationally symmetric reflectors that account for skew rays. Extension to fully freeform three-dimensional reflectors, stochastic optimization for scalability, and alternative loss formulations for discontinuous sources in higher dimensions remain open and promising directions.

\paragraph{Funding.} High Tech | TKI HSTM
\paragraph{Acknowledgment.} This work in the project MALIOD is funded by Holland High Tech | TKI HSTM via the PPS allowance scheme for public--private partnerships.
\paragraph{Disclosures.} The authors declare no conflicts of interest.
\paragraph{Data availability.} Data underlying the results presented in this paper are not publicly available at this time but may be obtained from the authors upon request.

\FloatBarrier

\bibliographystyle{iopart-num}
\bibliography{refs}

\newpage

\appendix

\section{Ray intersection guarantee}
\label{app:ray-intersection}

We show that the reflector parameterization of Section~\ref{sec:finite-source-problem} ensures that every emitted ray intersects the reflector curve, regardless of the choice of height function~$u$, provided only that $u$ is continuous and strictly positive.

Fix an arbitrary source point $s \in \Omega$ and consider the displacement from $(s, 0)$ to the reflector point $\mathbf{r}(p)$:
\begin{equation}
  \mathbf{r}(p) - \begin{bmatrix} s \\ 0 \end{bmatrix}
  \;=\;
  \begin{bmatrix}
    p - s + u(p)\cos\beta(p) \\[2pt]
    u(p)\sin\beta(p)
  \end{bmatrix}
  \;=:\;
  \begin{bmatrix} \Delta_x(s,p) \\[2pt] \Delta_z(p) \end{bmatrix}.
\end{equation}
We define the \emph{viewing angle} $\gamma_s(p) := \operatorname{atan2}\!\bigl(\Delta_z(p),\;\Delta_x(s,p)\bigr)$, i.e., the angle at which the source point $(s, 0)$ sees the reflector point $\mathbf{r}(p)$, measured counterclockwise from the positive $x$-axis. Since $u(p) > 0$ and $\beta(p) \in [\alpha_{\min}, \alpha_{\max}] \subset (0, \pi)$, the vertical component satisfies $\Delta_z(p) = u(p)\sin\beta(p) > 0$ for all $p \in \Omega$, and continuity of $u$ and $\beta$ ensures that $\gamma_s$ is continuous. We also note that, for any fixed $c > 0$, the function $\xi \mapsto \operatorname{atan2}(c, \xi)$ is strictly decreasing, as its derivative is $-c/(\xi^2 + c^2) < 0$.

At the left endpoint $p = L_{\min}$, we have $\beta(L_{\min}) = \alpha_{\max}$, giving $\Delta_z(L_{\min}) = u(L_{\min})\sin\alpha_{\max}$ and
\begin{equation}
  \Delta_x(s, L_{\min})
  \;=\;
  L_{\min} - s + u(L_{\min})\cos\alpha_{\max}
  \;\leq\;
  u(L_{\min})\cos\alpha_{\max},
\end{equation}
where the inequality follows from $s \geq L_{\min}$. When $s = L_{\min}$ there is equality, giving $\gamma_s(L_{\min}) = \alpha_{\max}$; for $s > L_{\min}$, $\Delta_x$ is strictly smaller while $\Delta_z$ is unchanged, so the monotonicity of $\operatorname{atan2}$ in its second argument gives $\gamma_s(L_{\min}) \geq \alpha_{\max}$. Analogously, at $p = L_{\max}$ we have $\beta(L_{\max}) = \alpha_{\min}$ and
\begin{equation}
  \Delta_x(s, L_{\max})
  \;=\;
  L_{\max} - s + u(L_{\max})\cos\alpha_{\min}
  \;\geq\;
  u(L_{\max})\cos\alpha_{\min},
\end{equation}
since $s \leq L_{\max}$, and the same argument yields $\gamma_s(L_{\max}) \leq \alpha_{\min}$.

Combining these bounds gives $[\alpha_{\min}, \alpha_{\max}] \subseteq [\gamma_s(L_{\max}),\; \gamma_s(L_{\min})]$. Since $\gamma_s$ is continuous on $[L_{\min}, L_{\max}]$, the Intermediate Value Theorem guarantees that for every $\alpha \in A$ there exists $p^{*} \in \Omega$ with $\gamma_s(p^{*}) = \alpha$, meaning
\begin{equation}
  \mathbf{r}(p^{*}) - \begin{bmatrix} s \\ 0 \end{bmatrix}
  \;=\;
  \lambda \begin{bmatrix} \cos\alpha \\ \sin\alpha \end{bmatrix}
\end{equation}
for some scalar $\lambda$. Since $\Delta_z(p^{*}) > 0$ and $\sin\alpha > 0$ (as $\alpha \in (0, \pi)$), we must have $\lambda > 0$, confirming that the ray from $(s, 0)$ in direction $(\cos\alpha, \sin\alpha)$ hits the reflector at $\mathbf{r}(p^{*})$. As $s$ and $\alpha$ were arbitrary, every emitted ray intersects the reflector curve.

\section{Scaling limit of the reflector}
\label{app:scaling-limit}

This appendix examines the behavior of the finite-source reflector problem when the height function is scaled uniformly, i.e., $u(p)$ is replaced by $\lambda u(p)$ for $\lambda > 0$. We show that, as $\lambda \to \infty$, the two-dimensional finite-source problem of Section~\ref{sec:finite-source-problem} reduces to a \emph{point-source-to-far-field reflector design problem}: the entire spatial extent of the source becomes invisible from the reflector, and the only quantity that matters is the angular marginal of the source distribution. Throughout, we assume $u \in C^1(\Omega)$ with $u > 0$ on $\Omega$, and write
\begin{equation}
  \mathbf{r}_\lambda(p)
  \;=\;
  \begin{bmatrix} p \\ 0 \end{bmatrix}
  \;+\;
  \lambda\, u(p)
  \begin{bmatrix} \cos\beta(p) \\ \sin\beta(p) \end{bmatrix}
\end{equation}
for the scaled reflector curve.

\begin{proposition}
\label{prop:scaling-limit}
Let $\sigma_\lambda(s,\alpha)$ denote the stereographic far-field coordinate produced by a ray emitted from $(s,0)$ at angle $\alpha$ and reflected by $\mathbf{r}_\lambda$. Assume that the limiting ray map $\sigma_\infty:A\to\Sigma$ takes values in a finite interval $\Sigma\subset\mathbb{R}$ and is a $C^1$ diffeomorphism. Then $\sigma_\lambda(s,\alpha)\to \sigma_\infty(\alpha)$ uniformly in $s$ as $\lambda\to\infty$, and the limiting far-field distribution satisfies
\begin{equation}
\label{eq:scaling-limit-result}
  g_\infty(\sigma)
  \;=\;
  F\!\bigl(\sigma_\infty^{-1}(\sigma)\bigr)\;
  \bigl|(\sigma_\infty^{-1})'(\sigma)\bigr|,
  \qquad
  F(\alpha) := \int_\Omega f(s,\alpha)\,\mathrm{d}s.
\end{equation}
\end{proposition}

\begin{proof}
The proof proceeds in three stages.

\medskip
\noindent\textit{Stage 1: Intersection parameter.}
The viewing angle from $(s, 0)$ to $\mathbf{r}_\lambda(p)$ is
\begin{equation}
  \gamma_s^{(\lambda)}(p)
  \;=\;
  \operatorname{atan2}\!\Bigl(
    u(p)\sin\beta(p),\;
    \tfrac{p - s}{\lambda} + u(p)\cos\beta(p)
  \Bigr),
\end{equation}
after dividing both arguments by $\lambda$. Since $\beta(p) \in [\alpha_{\min}, \alpha_{\max}] \subset (0, \pi)$ and $u > 0$, the first argument $u(p)\sin\beta(p)$ is bounded below by a positive constant on $\Omega$, so the argument pair remains bounded away from the origin uniformly in $s$, $p$, and $\lambda$. As $|p - s| \leq L_{\max} - L_{\min}$, the term $(p - s)/\lambda \to 0$ uniformly, and continuity of $\operatorname{atan2}$ away from the origin yields
\begin{equation}
\label{eq:gamma-limit-app}
  \gamma_s^{(\lambda)}(p) \;\to\; \beta(p),
  \qquad \text{uniformly in } s, p \in \Omega.
\end{equation}

Let $p_\lambda^*(s, \alpha)$ satisfy $\gamma_s^{(\lambda)}(p_\lambda^*) = \alpha$. The limiting equation $\beta(p) = \alpha$ has the unique solution $p_\infty^* = \beta^{-1}(\alpha)$, since $\beta$ is linear with nonzero slope $\beta' = (\alpha_{\min} - \alpha_{\max})/(L_{\max} - L_{\min})$. For any $\varepsilon > 0$ and $\lambda$ sufficiently large, $\|\gamma_s^{(\lambda)} - \beta\|_\infty < \varepsilon$, so that
\begin{equation}
  |\beta'|\,|p_\lambda^* - p_\infty^*|
  \;\leq\;
  |\beta(p_\lambda^*) - \gamma_s^{(\lambda)}(p_\lambda^*)|
  \;<\; \varepsilon,
\end{equation}
and thus $p_\lambda^*(s, \alpha) \to \beta^{-1}(\alpha)$ uniformly in $s$.

\medskip
\noindent\textit{Stage 2: Reflected direction.}
The unit incident direction from $(s, 0)$ to $\mathbf{r}_\lambda(p_\lambda^*)$ satisfies
\begin{equation}
  \hat{\mathbf{d}}_\lambda
  \;=\;
  \frac{1}{\|\mathbf{r}_\lambda(p_\lambda^*) - (s, 0)^\top\|}
  \begin{bmatrix}
    p_\lambda^* - s + \lambda u(p_\lambda^*)\cos\beta(p_\lambda^*) \\[3pt]
    \lambda u(p_\lambda^*)\sin\beta(p_\lambda^*)
  \end{bmatrix}.
\end{equation}
Dividing numerator and denominator by $\lambda$ and applying $p_\lambda^* \to \beta^{-1}(\alpha)$ gives $\hat{\mathbf{d}}_\lambda \to (\cos\alpha,\, \sin\alpha)^\top$, uniformly in $s$.

The tangent to $\mathbf{r}_\lambda$ at $p$ is
\begin{equation}
  \mathbf{r}_\lambda'(p)
  \;=\;
  \begin{bmatrix} 1 \\ 0 \end{bmatrix}
  + \lambda \left[
    u'(p) \begin{bmatrix} \cos\beta \\ \sin\beta \end{bmatrix}
    + u(p)\,\beta' \begin{bmatrix} -\sin\beta \\ \cos\beta \end{bmatrix}
  \right].
\end{equation}
Upon dividing by $\lambda$, the $(1, 0)^\top$ term vanishes in the limit, so the unit normal $\hat{\mathbf{n}}_\lambda(p_\lambda^*)$ converges to
\begin{equation}
\label{eq:limit-normal-app}
  \hat{\mathbf{n}}_\infty(\alpha)
  \;=\;
  \frac{1}{\sqrt{u'(p)^2 + u(p)^2 \beta'(p)^2}}
  \begin{bmatrix}
    u'(p)\sin\beta(p) + u(p)\beta'(p)\cos\beta(p) \\[3pt]
    -u'(p)\cos\beta(p) + u(p)\beta'(p)\sin\beta(p)
  \end{bmatrix}
  \bigg|_{p\,=\,\beta^{-1}(\alpha)}\!,
\end{equation}
where the denominator is strictly positive since $u > 0$ and $\beta' \neq 0$. The reflection formula $\mathbf{t} = \hat{\mathbf{d}} - 2\langle \hat{\mathbf{d}}, \hat{\mathbf{n}} \rangle \hat{\mathbf{n}}$ and the stereographic projection $\sigma = t_x/(1 - t_z)$ are both continuous in their arguments (the latter away from $t_z = 1$), so
\begin{equation}
\label{eq:sigma-limit-app}
  \sigma_\lambda(s, \alpha) \;\to\; \sigma_\infty(\alpha),
  \qquad \text{uniformly in } s \in \Omega.
\end{equation}

\medskip
\noindent\textit{Stage 3: Limiting distribution.}
Energy conservation gives, for any $\varphi \in C_c(\Sigma)$,
\begin{equation}
  \int_\Sigma g_\lambda(\sigma)\,\varphi(\sigma)\,\mathrm{d}\sigma
  \;=\;
  \int_\Omega \int_A f(s, \alpha)\,
  \varphi\!\bigl(\sigma_\lambda(s, \alpha)\bigr)\,\mathrm{d}\alpha\,\mathrm{d}s.
\end{equation}
By Eq.~\eqref{eq:sigma-limit-app}, the integrand converges pointwise to $f(s, \alpha)\,\varphi(\sigma_\infty(\alpha))$, and is dominated by the integrable function $|f(s, \alpha)|\,\|\varphi\|_\infty$, whose integral over $\Omega \times A$ is $\|f\|_{L^1(\mathcal{S})}\|\varphi\|_\infty < \infty$. The dominated convergence theorem and Fubini's theorem (exploiting the $s$-independence of $\sigma_\infty$) yield
\begin{equation}
  \lim_{\lambda \to \infty}
  \int_\Sigma g_\lambda(\sigma)\,\varphi(\sigma)\,\mathrm{d}\sigma
  \;=\;
  \int_A F(\alpha)\,\varphi\!\bigl(\sigma_\infty(\alpha)\bigr)\,\mathrm{d}\alpha.
\end{equation}
If $\sigma_\infty$ is a $C^1$ diffeomorphism, the substitution $\sigma = \sigma_\infty(\alpha)$ gives
\begin{equation}
  \int_A F(\alpha)\,\varphi\!\bigl(\sigma_\infty(\alpha)\bigr)\,\mathrm{d}\alpha
  \;=\; \int_\Sigma
    F\!\bigl(\sigma_\infty^{-1}(\sigma)\bigr)\,
    \bigl|(\sigma_\infty^{-1})'(\sigma)\bigr|\,
    \varphi(\sigma)\,\mathrm{d}\sigma.
\end{equation}
As this holds for all $\varphi \in C_c(\Sigma)$, the identity in Eq.~\eqref{eq:scaling-limit-result} follows.
\end{proof}

\noindent Eq.~\eqref{eq:scaling-limit-result} is precisely the energy-conservation relation for a \emph{two-dimensional point-source-to-far-field single-reflector} system: a point source at the origin with angular intensity $F(\alpha)$, and a reflector curve parameterized by a polar function $\rho(\alpha)$, designed so that the reflected light produces a prescribed far-field distribution $\hat{g}(\sigma)$. We now formulate this limiting design problem explicitly and show how it reduces to a pair of ODEs.

\begin{corollary}[Point-source design problem]
\label{cor:point-source-problem}
In the limit $\lambda \to \infty$, the finite-source design problem reduces to the following: given a point source at the origin with angular emission profile $F(\alpha) = \int_\Omega f(s, \alpha)\,\mathrm{d}s$ for $\alpha \in A$, find a reflector curve
\begin{equation}
\label{eq:polar-reflector}
  \mathbf{r}_\infty(\alpha) = \rho(\alpha) \begin{bmatrix} \cos\alpha \\ \sin\alpha \end{bmatrix}, \qquad \rho(\alpha) > 0,
\end{equation}
such that the far-field distribution after reflection equals a prescribed $\hat{g}(\sigma)$.
\end{corollary}

\noindent The relationship between Eq.~\eqref{eq:polar-reflector} and the original parameterization is $\rho(\alpha) = u(\beta^{-1}(\alpha))$, with the spoke direction $\beta(p)$ evaluated at $p = \beta^{-1}(\alpha)$ reducing to $\alpha$ itself. The base-line offset $(p, 0)^\top$ in the original parameterization becomes negligible relative to $\lambda u$ as $\lambda \to \infty$, so the reflector geometry is that of a polar curve centered at the source.

This point-source problem is solved by two sequential ODEs.

\paragraph{ODE 1: Ray map.}
The monotonic ray map $\sigma_\infty : A \to \Sigma$ achieving the prescribed $\hat{g}(\sigma)$ is determined by the flux-balance condition $\int_{\alpha_{\min}}^{\alpha} F\,\mathrm{d}\alpha' = \int_{T_{\min}}^{\sigma_\infty(\alpha)} \hat{g}\,\mathrm{d}\sigma$, which upon differentiation gives
\begin{equation}
\label{eq:limit-ray-map-ode}
  \sigma_\infty'(\alpha)
  = \frac{F(\alpha)}{\hat{g}\!\bigl(\sigma_\infty(\alpha)\bigr)},
  \qquad
  \sigma_\infty(\alpha_{\min}) = T_{\min}.
\end{equation}

\paragraph{ODE 2: Reflector profile.}
Once the ray map $\sigma_\infty(\alpha)$ is known, the reflector profile $\rho(\alpha)$ is recovered from the law of reflection. From the chain rule, $u' = \rho'\beta'$, and substituting into Eq.~\eqref{eq:limit-normal-app}, the factors of $|\beta'|$ cancel. The limiting normal in terms of $\rho$ and $\rho'$ is
\begin{equation}
  \hat{\mathbf{n}}_\infty(\alpha)
  = \frac{1}{\sqrt{\rho'^2 + \rho^2}}
  \begin{bmatrix}
    \rho'\sin\alpha + \rho\cos\alpha \\[2pt]
    -\rho'\cos\alpha + \rho\sin\alpha
  \end{bmatrix}.
\end{equation}
With $\hat{\mathbf{d}} = (\cos\alpha, \sin\alpha)^\top$, one computes $\langle \hat{\mathbf{d}}, \hat{\mathbf{n}}_\infty \rangle = \rho / \sqrt{\rho'^2 + \rho^2}$, and the reflected direction becomes
\begin{equation}
\label{eq:limit-reflected}
  \mathbf{t}_\infty
  = \begin{bmatrix} \cos\alpha \\ \sin\alpha \end{bmatrix}
    - \frac{2\rho}{\rho'^2 + \rho^2}
    \begin{bmatrix} \rho'\sin\alpha + \rho\cos\alpha \\[2pt] -\rho'\cos\alpha + \rho\sin\alpha \end{bmatrix}.
\end{equation}
Imposing $t_x/(1 - t_z) = \sigma_\infty(\alpha)$ and solving for $\rho'$ yields a first-order ODE $\rho' = \Phi(\alpha, \rho; \sigma_\infty)$, to be integrated from $\rho(\alpha_{\min}) = h$ for a chosen $h > 0$. This is structurally identical to the integrating-factor approach of Eq.~\eqref{eq:approx-sol-expr}.

\paragraph{Comparison with the approximate problem.}
The scaling limit and the approximate problem of Section~\ref{sec:finite-source-approx-problem} both reduce the finite-source problem to a one-dimensional ODE, but differ in two respects. First, the effective source distributions differ: the approximate problem uses the spatial marginal $f(s) = \int_A f(s, \alpha)\,\mathrm{d}\alpha$, whereas the scaling limit uses the angular marginal $F(\alpha) = \int_\Omega f(s, \alpha)\,\mathrm{d}s$. For a separable source $f(s, \alpha) = f_s(s)\,f_\alpha(\alpha)$, these are proportional to $f_s$ and $f_\alpha$ respectively. Second, the ODE coefficients differ: at finite height, the tangent $\mathbf{r}'(p)$ retains the base-line term $(1, 0)^\top$, which enters the normal and modifies the reflection geometry. In the scaling limit, this term vanishes and the geometry becomes that of a polar curve centered at the source.

\paragraph{Method selection.}
The scaling limit also gives a concrete criterion for when the finite-source methods of this paper are needed at all. When the reflector is far enough from the source that the source's angular footprint at the reflector is small compared with the angular range $[\alpha_{\min}, \alpha_{\max}]$, classical point-source and nonimaging-optics constructions---the polar-reflector ODE above, SMS designs\,\cite{Benitez2004SMS3D,Winston2004Nonimaging}, and optimal-transport solvers for the point-source Monge--Amp\`ere problem\,\cite{GlimmOliker2003JMS,GangboOliker2007COCV,Romijn2019JOSAA,Romijn2020JCP}---are appropriate, accurate, and considerably cheaper than either method developed here. When the source's angular extent is comparable to the angular footprint of the reflector, neither the scaling limit nor a single one-dimensional ODE in the spirit of Section~\ref{sec:finite-source-approx-problem} captures the full mapping, and a finite-source method such as the neural network of Section~\ref{sec:ann} or the iterative deconvolution of Section~\ref{sec:iterative-deconv} is required.

\end{document}